\documentclass[twocolumn,journal]{IEEEtran}

\newcommand{\ourmethod}{\text{UOWQ}}
\newcommand{\allsource}{\text{AllSources $\cup$ Target}}
 
\IEEEoverridecommandlockouts

\usepackage[T1]{fontenc} 
\usepackage{cite}
\usepackage{amsmath,amssymb,amsfonts}
\usepackage{algorithm}
\usepackage{algorithmic}
\usepackage{graphicx}
\usepackage{textcomp}
\usepackage{xcolor}

\usepackage{pifont}
\usepackage{bbding}

\usepackage{multirow}
\usepackage{booktabs}
\usepackage{amsthm}  

\DeclareMathAlphabet{\mathbbb}{U}{bbold}{m}{n}  

\newtheorem{theorem}{Theorem}
\newtheorem{lemma}[theorem]{Lemma}
\newtheorem{proposition}[theorem]{Proposition}

\newtheorem{definition}[theorem]{Definition}
\newtheorem{remark}[theorem]{Remark}

\DeclareMathOperator*{\argmax}{arg\,max}
\DeclareMathOperator*{\argmin}{arg\,min}

\newcommand\tsup[2][2]{%
	\def\useanchorwidth{T}%
	\ifnum#1>1%
	\stackon[-1.3ex]{\tsup[\numexpr#1-1\relax]{#2}}{\mathchar"307E\kern-.5pt}%
	\else%
	\stackon[-1ex]{#2}{\mathchar"307E\kern-.5pt}%
	\fi%
}

\newcommand{\cT}{{\mathcal{T}}}

\newcommand{\X}{{\mathcal{X}}}

\newcommand{\cA}{\mathcal{A}}

\newcommand{\cL}{\mathcal{L}}

\newcommand{\cS}{\mathcal{S}}

\newcommand{\cX}{\X}

\newcommand{\defeq}{\triangleq}

\usepackage{lipsum}
\usepackage{caption}
\newcommand{\mycomment}[1]{\hspace{\fill} \textcolor{black}{// #1}}
\def\BibTeX{{\rm B\kern-.05em{\sc i\kern-.025em b}\kern-.08em
    T\kern-.1667em\lower.7ex\hbox{E}\kern-.125emX}}

\usepackage{orcidlink}
\begin{document}

\title{
Unified Optimization of Source Weights and Transfer Quantities in Multi-Source Transfer Learning: An Asymptotic Framework
}

\author{
Qingyue Zhang\orcidlink{0000-0001-9197-1141}, Chang Chu\orcidlink{0009-0000-0958-0046}, Haohao Fu\orcidlink{0009-0005-9968-4098}, Tianren Peng\orcidlink{0000-0002-1779-1633}, Yanru Wu\orcidlink{0009-0008-3738-5019}, Guanbo Huang\orcidlink{0009-0001-1284-9787}, Yang Li\orcidlink{0000-0002-2053-6393}, and Shao-Lun Huang\orcidlink{0000-0003-2827-4022} %
\thanks{The authors are with the
Tsinghua Shenzhen International Graduate School, Tsinghua University, Shenzhen, China. Yang Li is now with the School of Artificial Intelligence,
The Chinese University of Hong Kong, Shenzhen, China. 
This work was done while she was with the Tsinghua Shenzhen International Graduate School,
Tsinghua University, Shenzhen, China.}%
\thanks{This work was supported in part by the National Key R\&D Program of China
under Grant 2021YFA0715202, the National Natural Science Foundation of China  
under Grants 62571296 and 62371270, and the Shenzhen Science and Technology Program under Grant KJZD20240903102700001.}%
\thanks{Corresponding authors: Yang Li (yangl@cuhk.edu.cn), Shao-Lun Huang (twn2gold@gmail.com).}%
\thanks{Our code is publicly available at:
\url{https://github.com/zqy0126/UOWQ}.}%

}

\maketitle

\begin{abstract}
In multi-source transfer learning, a key challenge lies in how to appropriately differentiate and utilize heterogeneous source tasks.
However, existing multi-source methods typically focus on optimizing either the source weights or the amount of transferred samples, largely neglecting their joint consideration.
In this work, we propose a theoretical framework, \textbf{U}nified \textbf{O}ptimization of \textbf{W}eights and \textbf{Q}uantities (UOWQ), that jointly determines the optimal source weights and transfer quantities for each source task. Specifically, the framework formulates multi-source transfer learning as a parameter estimation problem based on an asymptotic analysis of a Kullback--Leibler divergence--based generalization error measure, leading to two main theoretical findings: 
1) using all available source samples is always optimal when the weights are properly adjusted; 
2) the optimal source weights are characterized by a principled optimization problem whose structure explicitly incorporates the Fisher information, parameter discrepancy, parameter dimensionality, and transfer quantities.
Building on the theoretical results, we further propose a practical algorithm for multi-source transfer learning, and extend it to multi-task learning settings where each task simultaneously serves as both a source and a target.
Extensive experiments on real-world benchmarks, including DomainNet and Office-Home, demonstrate that \ourmethod{} consistently outperforms strong baselines. The results validate both the theoretical predictions and the practical effectiveness of our framework.

\end{abstract}


\section{Introduction}
\label{section_Introduction}

Transfer learning has emerged as a critical paradigm in modern machine learning, especially in scenarios where labeled data is scarce for the target task~\cite{xiao2024selective}. By leveraging knowledge from related source tasks, transfer learning enables models to generalize more effectively and efficiently to the target task. In recent years, transfer learning has a wide range of applications. For example, in deep reinforcement learning, transfer learning plays a crucial role in accelerating policy learning and improving sample efficiency by transferring policies, representations, or experience \cite{tpami2023_zhu_Reinforcelearning_transfer}.
In healthcare, transfer learning leverages large medical datasets to support rare disease diagnosis with limited clinical data \cite{serte2022deep}. In natural language processing, pretrained models built on large-scale transfer learning foundations have revolutionized performance across a wide range of downstream tasks \cite{devlin2019bert}. 
While single-source transfer learning has been studied earlier, growing attention is now being paid to multi-source transfer learning.
With the rapid development of large language models (LLMs) trained on massive and diverse corpora, effectively leveraging complementary knowledge from multiple heterogeneous data sources has become increasingly important.


\begin{table*}[t]
\setlength{\tabcolsep}{4pt}
    \renewcommand{\arraystretch}{0.76}
    \centering
    \caption{
    \textbf{Comparison across matching-based transfer learning methods,}
based on whether they perform quantity optimization, weight optimization, 
are tailored to multi-source, have task generality, have shot generality, 
and require target labels. 
The `\Checkmark' represents obtaining the corresponding aspects, while `\XSolidBrush' represents the opposite. 
Quantity optimization denotes the ability to determine how many source samples or data subsets should be transferred. 
Weight optimization denotes the ability to assign or learn weights for different source tasks or samples. 
Task generality denotes the ability to handle various target task types, 
and shot generality denotes the ability to avoid negative transfer in different target sample quantity settings, 
including few-shot and non-few-shot.
    }
    \begin{tabular}{c c c c c c c}
        \toprule
         \textbf{Method} &  
         \textbf{Quantity Optimization} &
         \textbf{Weight
         Optimization} &
         \textbf{Multi-Source}& \textbf{Task Generality}  & \textbf{Shot Generality} & \textbf{Target Label} \\
         \midrule
         MCW \cite{lee2019learning_MCW} &  \XSolidBrush & \Checkmark &  \Checkmark & \XSolidBrush & \XSolidBrush & Supervised \\
         Leep \cite{nguyen2020leep} &  \XSolidBrush & \XSolidBrush &  \XSolidBrush &  \Checkmark & \Checkmark & Supervised \\
         Tong \cite{tong2021mathematical} &  \XSolidBrush & \Checkmark &  \Checkmark &  \XSolidBrush & \Checkmark &Supervised \\
         DATE \cite{han2023discriminability_DATE} &  \XSolidBrush & \Checkmark & \Checkmark   &  \Checkmark & \Checkmark & Unsupervised \\
         H-ensemble \cite{wu2024h_Hensemble}  &  \XSolidBrush & \Checkmark &  \Checkmark & \XSolidBrush & \XSolidBrush & Supervised \\ 
          FOMA \cite{li2024scalable} &\Checkmark&\XSolidBrush&   \Checkmark & \XSolidBrush & \XSolidBrush & Supervised \\ 
          DBF \cite{jain2023data}  &\Checkmark&\XSolidBrush&  \XSolidBrush & \XSolidBrush & \XSolidBrush & Supervised \\ 
          OTQMS \cite{zhang2025theoretical}  &\Checkmark&\XSolidBrush&  \Checkmark & \Checkmark & \Checkmark & Supervised \\ 
         \ourmethod{} (Ours)   & \Checkmark & \Checkmark &  \Checkmark &  \Checkmark &  \Checkmark & Supervised \\
         \bottomrule
    \end{tabular}
    \label{tab:comparison}
\end{table*}



In multi-source transfer learning, how to differentially leverage multiple source tasks has emerged as a central research problem. From a statistical perspective, while incorporating additional data from source tasks can reduce estimation variance, sources that are overly dissimilar to the target task may introduce excessive systematic bias, outweighing the benefits of variance reduction. As a result, naive uniform transfer from multiple source tasks can lead to negative transfer~\cite{wang2019characterizing}, making the identification and effective utilization of truly related source tasks a key challenge.
Existing approaches typically address this challenge from one of two separate perspectives. On one hand, \textbf{task- or source-weighting} methods aim to control the influence of different sources by assigning importance weights to source models or samples~\cite{shui2021aggregating_WADN,wu2024h_Hensemble}, thereby mitigating bias caused by mismatched domains. 
In existing studies on task weighting, most frameworks implicitly assume that all source samples are utilized, while the transfer weights are determined independently of the transfer quantity.
On the other hand, \textbf{sample selection} methods focus on determining the optimal transfer quantity for each source and identifying the source samples to be transferred~\cite{jain2023data,li2024scalable, zhang2025theoretical}.
Considering source weighting and transfer quantities in isolation restricts the solution space and can lead to suboptimal solutions, whereas jointly optimizing them leads to more principled and effective solutions.
The goal of this work is to establish a unified theoretical framework that not only jointly determines the optimal source weights and transfer quantities, but also provides a principled interpretation of their interaction, thereby offering insights into existing weighting and sample selection strategies.
In Table \ref{tab:comparison}, we provide a detailed comparison of the proposed method with several existing methods.    


In this work, we develop a theoretical framework termed \textbf{U}nified \textbf{O}ptimization of \textbf{W}eights and \textbf{Q}uantities (UOWQ), to determine the optimal weights and quantities for source tasks in transfer learning via asymptotic analysis.
Through an asymptotic analysis of the measure based on Kullback--Leibler (K-L) divergence, we formulate the multi-source transfer learning problem as a parameter estimation problem, which in turn yields an optimization problem of the source weights and transfer quantities. 
First, we prove the conclusion that using all available source samples is always optimal once the weights are properly adjusted, and provide an intuitive interpretation of this result.
Moreover, our analysis provides a closed-form solution for the optimal transfer weight in the single-source setting and characterizes the optimal weights in the multi-source setting via a convex optimization formulation, jointly accounting for the Fisher information, parameter discrepancy, parameter dimensionality, and transfer quantities.
On this basis, we introduce practical algorithms capable of supporting both multi-source transfer learning and multi-task transfer learning.  
Finally, the effectiveness of the proposed algorithms is validated through experiments on real-world datasets. 

The main contributions of this work are summarized as follows:  

\begin{itemize}
    \item \textbf{Theoretical framework.} 
    We introduce \ourmethod{}, a mathematical framework based on asymptotic analysis, to jointly optimize transfer quantities and source weights in multi-source transfer learning. Under joint optimization, we prove that using all available source samples is always optimal. 
    In addition, we characterize the optimal transfer weights for each source via an optimization formulation that explicitly accounts for the transfer quantities.
    \item \textbf{Practical algorithm.} Based on our theory, we develop practical algorithms derived from the framework for both multi-source transfer learning and multi-task learning. 
    Specifically, to cope with scarce target data and ensure robustness, the algorithms alternate between model optimization and source weight updates. 
    \item \textbf{Experimental validation.}
    We perform extensive experiments on the \texttt{DomainNet} and \texttt{Office-Home} benchmarks under both multi-source transfer learning and multi-task learning settings.
    In the 10-shot multi-source transfer scenario, \ourmethod{} consistently outperforms strong baselines, achieving average accuracy improvements of 1.3\% on DomainNet and 1.4\% on Office-Home.
    Moreover, in the multi-task learning setting, \ourmethod{} surpasses state-of-the-art task-weighting methods by 0.7\% on \texttt{DomainNet} and 0.4\% on \texttt{Office-Home} in terms of overall average performance.
    In addition, we conduct comprehensive ablation and robustness studies, including analyses under varying target-shot regimes, weight visualization, and computational efficiency evaluations, to further validate the theoretical insights and practical robustness of \ourmethod{}.

\end{itemize}


This work is an extension of our previous conference paper~\cite{zhang2025apwdsit}, and the additional contributions are summarized as follows. 
First, the journal version presents a more explicit formulation for jointly optimizing transfer quantities and source weights, together with complete theoretical proofs. 
In addition, we provide further theoretical analysis, including an intuitive interpretation of why utilizing all source samples remains optimal under joint optimization in Section~\ref{subsection_multi-Source Transfer Learning}, as well as comparisons with related theoretical frameworks in Remark~\ref{remark_Connection to Prior Work} and Table~\ref{tab:comparison}.
Second, in the practical algorithm part in Section~\ref{section_Practical_Algorithms}, we further develop a multi-task learning algorithm based on the proposed multi-source transfer learning algorithm.
Finally, the experimental section is substantially strengthened with larger-scale datasets, more comprehensive baselines, and evaluations of the proposed multi-task learning extension, along with extensive analyses including ablation studies, weight visualization, robustness evaluation, computational efficiency analysis, and compatibility with Low-Rank Adaptation (LoRA).

The remainder of this paper is organized as follows. Section~\ref{related_work} reviews the related work. Section~\ref{section_Preliminaries} introduces several preliminary results used in our analysis and presents the problem formulation. Section~\ref{section_main_result} develops the main theoretical results to calculate the optimal source weights. Section~\ref{section_Practical_Algorithms} introduces practical algorithms. Section~\ref{section_Experiments} empirically evaluates the proposed framework and validates the theoretical findings.

\section{Related Work} 
\label{related_work}
\subsection{Multi-source Transfer Learning}

Multi-source transfer learning (MSTL) leverages knowledge from multiple related source domains or tasks to improve performance on a target task, addressing challenges such as domain shift and negative transfer~\cite{liu2021task,wang2020multisource,zhou2021deep}. By emphasizing relevant sources while suppressing less informative ones, MSTL enhances generalization in heterogeneous settings, where effective learning requires balancing source contributions and mitigating distribution discrepancies.

From the perspective of the \emph{transfer object}, multi-source transfer learning methods can be broadly categorized into \textbf{model-based} and \textbf{sample-based} transfer~\cite{zhuang2020comprehensive}. Model-based approaches leverage pretrained source models through fine-tuning or parameter adaptation~\cite{wan2022uav}, whereas sample-based methods jointly train on target and weighted source samples, enabling more direct exploitation of task-relevant source data~\cite{zhang2024revisiting_MADA,shui2021aggregating_WADN,li2021dynamic}.

From the perspective of the \emph{transfer strategy}, existing methods can be further divided into \textbf{alignment-based} and \textbf{matching-based} approaches~\cite{zhao2024more}. Alignment-based methods reduce domain discrepancy by explicitly aligning feature distributions~\cite{li2021multi,zhao2021madan,li2021dynamic}, while matching-based methods emphasize informative sources or samples via selective weighting~\cite{guo2020multi,shui2021aggregating_WADN,tong2021mathematical,wu2024h_Hensemble}. Our method is more closely aligned with the \textbf{sample-based, matching-based} category, as it adopts a unified framework that jointly optimizes transfer quantities and source weights for weighted joint training.

\subsection{Task weighting}

Task weighting, also referred to as source weighting, is a representative class of \textbf{matching-based} approaches in multi-source transfer learning and multi-task learning, as it regulates the contribution of different tasks without explicitly aligning feature distributions. Its objective is to balance competing task objectives during joint optimization. Early approaches relied on \textbf{static weighting}, assigning fixed weights based on heuristics such as task priority or dataset size \cite{caruana1997multitask}, but these methods lack adaptability to dynamic task interactions. Recent studies have therefore focused on \textbf{dynamic task weighting} strategies. For example, uncertainty-based weighting was introduced in \cite{kendall2018multi}, GradNorm balances tasks by normalizing gradient magnitudes \cite{chen2018gradnorm}, and DWA adjusts weights according to task learning speeds \cite{liu2019end}. From a multi-objective optimization perspective, MGDA \cite{sener2018multi} formulates task weighting as a Pareto-optimal optimization problem.

Recent work has further extended task weighting to large-scale pretrained models, primarily through gradient-based matching mechanisms. For instance, task reweighting via gradient alignment has been studied in \cite{zheng2021libra}, while conflicting gradients are mitigated by CAGrad \cite{liu2021conflict}. Other studies explore \textbf{meta-learning}–based strategies \cite{sun2020mtadam} and task affinity to automate weighting. Despite these advances, task weighting remains challenged by nonstationary task relationships \cite{maninis2022rotograd}, scalability to large task sets, and the lack of strong theoretical guarantees \cite{du2023game,chen2023flix}. Benchmark datasets such as \textbf{Meta-Dataset} \cite{triantafillou2020meta} have been introduced to facilitate standardized evaluation.

Most existing task-weighting methods operate at the model level, assigning scalar weights to source models in a mixture-of-experts fashion \cite{mansour2008domain}. These approaches typically rely on task similarity \cite{long2015learning}, domain divergence \cite{zhao2018adversarial}, or validation performance \cite{chen2018domain}. In contrast, only a limited number of studies consider sample-level weighting, where source samples are reweighted and jointly trained with the target data in a unified framework~ \cite{shui2021aggregating_WADN,zhang2024revisiting_MADA}.
Moreover, most existing task-weighting methods implicitly assume access to all source samples, and the resulting weights are independent of the transfer quantity. While recent works have begun to study the problem of optimally selecting transferable source samples or transfer quantities~\cite{jain2023data,li2024scalable,zhang2025theoretical}, these efforts are largely decoupled from task-weight optimization. Our work provides a unified treatment that jointly optimizes task weights and transfer quantities, enabling principled control over both which sources to transfer from and how much to transfer.

\subsection{Transfer Learning Theory}

Existing theoretical studies in transfer learning can generally be divided into two primary categories. The first category aims to define and quantify the similarity or relatedness between source and target tasks. Numerous metrics have been proposed in this context, including the $l_2$ distance~\cite{long2014transfer}, optimal transport cost~\cite{courty2016optimal}, LEEP (Log Expected Empirical Prediction)~\cite{nguyen2020leep}, Wasserstein distance~\cite{shui2021aggregating_WADN}, OTCE (Optimal Transport-based Conditional Entropy)~\cite{chen2022otce}, LogME \cite{you2021logme}, NCE\cite{tan2020survey}, GBC (Geometric-Based Correlation) \cite{liu2021geometric}, and maximal correlation-based measures~\cite{lee2019learning_MCW}. For example, LEEP \cite{nguyen2020leep} provides a probabilistic framework for transferability estimation by comparing model predictions on the target task. Meanwhile, OTCE \cite{chen2022otce} leverages optimal transport theory to quantify transferability under distribution shifts.  
These measures provide principled ways to assess how well knowledge from a source task can potentially benefit a target task. 

This work belongs to the second group, which is dedicated to developing theoretical measures that assess and bound the generalization error in transfer learning scenarios.
Various generalization error measures have been proposed to guide the assignment of task or source weights, including $f$-divergence~\cite{harremoes2011pairs}, mutual information 
~\cite{bu2020tightening}, $\mathcal{H}$-score~\cite{wu2024h_Hensemble}, and $\X^2$-divergence~\cite{tong2021mathematical}.
Specifically, the $\mathcal{H}$-score analytically characterizes the expected log-loss on the target task under a fixed source feature representation, thereby providing a principled basis for evaluating target predictive performance in transfer learning.
Furthermore, a recently developed generalization measure based on K-L divergence~\cite{zhang2025theoretical} provides a rigorous foundation for determining optimal transfer quantities. In this work, we demonstrate that such a measure offers a more direct alignment with the cross-entropy loss---a standard objective in machine learning---than previous metrics. Consequently, we extend this metric and framework to jointly optimize source weights and transfer quantities.

\section{Preliminaries}
\label{section_Preliminaries}

In this section, we introduce the theoretical foundations of our framework. 
We first present the K-L divergence--based generalization error measure in Section~\ref{subsection_K-L divergence based measure}, 
which serves as the optimization objective. 
We then review the asymptotic normality of the MLE in Section~\ref{subsection_Asymptotic Normality of the MLE}, 
which enables characterizing estimation variance. 
Finally, we formulate the multi-source transfer learning problem and its joint optimization objective in Section~\ref{subsection_Problem Formulation}.


\subsection{K-L Divergence-Based Measure}
\label{subsection_K-L divergence based measure}
We introduce a K-L divergence-based measure as the generalization error measure. 

\begin{definition}[The K-L divergence \cite{cover1999elements}]
The K-L divergence $D\left(P \middle\| Q\right)$ measures the difference between two probability distributions \( P(X) \) and \( Q(X) \) over the same probability space. It is defined as:
\begin{align}
    D\left(P \middle\| Q\right) = \sum_{x \in \mathcal{X}} p(x) \log \frac{p(x)}{q(x)}.\notag
\end{align}
\end{definition}

In this work, we use the expectation of K-L divergence between the true distribution of target task $P_{X;{\underline{\theta}}_0}$ and the distribution $P_{X;\hat{{\underline{\theta}}}}$ learned from training samples as the generalization error measure, i.e., 
\begin{align}
\label{eq:KL1}
\mathbb{E}\left[D(P_{X;{\underline{\theta}}_0}||P_{X;\hat{{\underline{\theta}}}})\right]
.\end{align}
Compared to other measures, the K-L divergence exhibits a closer correspondence with the generalization error measured by the cross-entropy loss; we formally justify this relationship in the Appendix.

\subsection{Asymptotic Normality of the MLE}
\label{subsection_Asymptotic Normality of the MLE}

When attempting to recover the true parameter vector $ \underline{\theta}^* $ from independent and identically distributed (i.i.d.) observations generated according to the distribution $ P_{X;\underline{\theta}^*} $, we let $ \mathcal{D} $ denote a collection of $ n $ such i.i.d. samples. The maximum likelihood estimator (MLE) is then defined as the maximizer of the empirical log-likelihood:  
\begin{align}
\label{mle_target}
\hat{{\underline{\theta}}}_{\mathrm{MLE}}  
    = \argmax_{{\underline{\theta}}} \; \frac{1}{n} \sum_{x \in \mathcal{D}} \log P_{X;{\underline{\theta}}}(x).
\end{align}

Provided that the underlying distribution satisfies the standard regularity assumptions, the MLE exhibits the well-known property of \textbf{asymptotic normality}~\cite{van2000asymptotic}. Specifically, as the sample size increases, the distribution of the normalized estimation error converges in law to a multivariate Gaussian distribution:
\begin{align}
\label{eq:asymptotic_normality}
\sqrt{n}\left( \hat{{\underline{\theta}}}_{\mathrm{MLE}} - \underline{\theta}^* \right) 
    \xrightarrow{d} \mathcal{N}\!\left( 0, J(\underline{\theta}^*)^{-1} \right),
\end{align}
where the notation ``${-1}$'' indicates the matrix inverse, and $ J(\underline{\theta}) $ denotes the Fisher information matrix~\cite{cover1999elements}. The Fisher information matrix, which characterizes the amount of information carried by the distribution about the parameter, is defined as
\begin{align}
    J({\underline{\theta}})^{d \times d} 
    &= \mathbb{E} \left[ 
        \left( \frac{\partial}{\partial {\underline{\theta}}} \log P_{X;{\underline{\theta}}} \right) 
        \left( \frac{\partial}{\partial {\underline{\theta}}} \log P_{X;{\underline{\theta}}} \right)^{\!T} 
    \right].
\end{align}
Intuitively, the Fisher information matrix characterizes the local geometry of the log-likelihood function with respect to the parameter $\underline{\theta}$. It captures how sensitively the distribution responds to perturbations along different parameter directions, thereby reflecting the statistical identifiability of $\underline{\theta}$.

\subsection{Problem Formulation}
\label{subsection_Problem Formulation}

A multi-source transfer learning framework consists of a target 
task $\mathcal{T}$ and multiple source tasks 
$\{\mathcal{S}_1, \dots, \mathcal{S}_K\}$ that provide auxiliary 
information to improve performance on $\mathcal{T}$. 
To generalize the analysis, we model $\mathcal{T}$ as a parameter 
estimation problem governed by an underlying distribution 
$P_{X;\underline{\theta}}$, where $\underline{\theta}$ denotes 
the parameter of interest and $X$ represents a generic random 
variable corresponding to the data. For example, in supervised 
learning, $X=(Z,Y)$ represents the joint distribution over input 
features $Z$ and labels $Y$. The objective is to accurately 
estimate the true value of $\underline{\theta}$, which in deep 
learning corresponds to optimizing the network parameters of 
$\mathcal{T}$. Furthermore, we assume that the source tasks and the target task follow the same parametric model and share the same input space $\mathcal{X}$.
For notational clarity, we present only the case where $\mathcal{X}$ is discrete, while the theoretical results extend straightforwardly to continuous domains.

The target task $\mathcal{T}$ has $N_0$ i.i.d.\ training samples 
drawn from $P_{X;\underline{\theta}_0}$, where 
$\underline{\theta}_0 \in \mathbb{R}^d$, and the training set is 
denoted by $X^{N_0}$. Similarly, each source task 
$\mathcal{S}_i$ has $N_i$ i.i.d.\ samples from 
$P_{X;\underline{\theta}_i}$, where 
$i \in [1, K]$ and $\underline{\theta}_i \in \mathbb{R}^d$, 
denoted by $X^{N_i}$.
During training, each sample is weighted in the gradient descent 
procedure according to the source weight associated with its 
corresponding task. Motivated by this weighting mechanism, we 
formulate the training process as a parameter estimation problem 
and define the corresponding estimator.
Specifically, $\hat{{\underline{\theta}}}$ is denoted as the MLE based on the $N_0$ samples from $\cT$ and $n_1,\dots,n_K$ samples from $\cS_1,\dots,\cS_K$ with weights $w_1,\dots,w_K$, where $n_i \in [0,N_i]$ and $w_i \in [0,+\infty]$, i.e.,
\begin{align}
\label{mle_defination}
    \hat{{\underline{\theta}}} = \argmax_{{\underline{\theta}}} &{\sum_{x \in X^{N_{0}}}} \log P_{X;{\underline{\theta}}}(x)+{\sum_{i=1}^{K}\sum\limits_{x \in X^{n_{i}}}} w_i\log P_{X;{\underline{\theta}}}(x),
\end{align}
where we require $w_i \ge 0$ because negative weights would subtract sample
likelihoods, contradict the statistical meaning of MLE and turn the
weighted log-likelihood into a non-concave and unbounded objective.

In this work, 
our goal is to derive the optimal transfer weight $w_1^*,\dots,w_K^*$ and transfer quantities $n_1^*,\dots,n_K^*$ of source tasks $\cS_1,\dots,\cS_K$ to minimize the K-L based measure between the true distribution of target task $P_{X;{\underline{\theta}}_0}$ and the distribution $P_{X;\hat{{\underline{\theta}}}}$ learned from training samples, \textit{i.e.},
\begin{align}
\label{optimization_problem}
&w_1^*,\dots,w_K^*,n_1^*,\dots,n_K^*=\notag\\
&\argmin_{w_1,\dots,w_k,n_1,\dots,n_K} \mathbb{E}\left[D(P_{X;{\underline{\theta}}_0}||P_{X;\hat{{\underline{\theta}}}})\right].
\end{align}

For theoretical analysis, we assume a compact parameter space 
$\Theta$ and standard regularity conditions on 
$P_{X;\underline{\theta}}$, including differentiability, 
parameter-independent support, and finite, positive Fisher 
information.

\section{Theoretical Characterization of Optimal Source Weights and Transfer Quantities}
\label{section_main_result}


In this section, we present the main theoretical results of \ourmethod{}. 
We first characterize the optimal transfer quantity and the optimal source weight in the single-source setting, 
deriving closed-form expressions
in Section~\ref{subsection_Single-Source Transfer Learning}. 
We then extend the analysis to the multi-source setting, 
where the joint optimization problem admits a convex formulation and a tractable numerical solution in Section~\ref{subsection_multi-Source Transfer Learning}.


\subsection{Single-Source Transfer Learning}
\label{subsection_Single-Source Transfer Learning}

In this subsection, we first analyze the case where the parameter is one-dimensional in Theorem \ref{thm:one_source}. We then extend the method to the high-dimensional setting in Proposition \ref{prop:highdimension}. Throughout this paper, we denote scalar-valued parameters by $\theta$ and high-dimensional parameters by $\underline{\theta}$. In analyzing the target parameter estimation problem, directly computing the K-L measure is intractable in general. Therefore, we conduct our analysis primarily in the asymptotic regime, where the behavior of the estimator becomes more tractable. To begin with, we consider the transfer learning scenario where we have one target task $\cT$ with $N_0$ training samples and one source task $\cS_1$ with $N_1$ training samples. In this case, we aim to determine the optimal transfer quantity $n_1^*\in [1,N_1]$ and the optimal transfer weight $w_1^*\in [0,+\infty]$. To facilitate our mathematical derivations, we assume 
$N_0$ and $N_1$ are asymptotically comparable, and
the distance between the parameters of the target task and source task is sufficiently small (\textit{i.e.}, $\vert\theta_0-\theta_1\vert=O(\frac{1}{\sqrt{N_{0}}})$). Considering the similarity of low-level features among tasks of the same type~\cite{raghu2019transfusion},
this assumption is made without loss of generality, and similar assumptions have also been adopted in the prior literature~\cite{zhang2025theoretical}.

\begin{theorem}
\label{thm:one_source}(proved in the Appendix)
In the single-source setting with 1-dimensional models $P_{X;\theta_0}$ and $P_{X;\theta_1}$, we assume that $\theta_0,\theta_1 \in \mathbb{R}$ and $\vert\theta_0-\theta_1\vert=O(\frac{1}{\sqrt{N_{0}}})$. 
Then, the K-L measure $\mathbb{E}[D(P_{X;{\underline{\theta}}_0}||P_{X;\hat{{\underline{\theta}}}})]$ can be expressed as: 
    \begin{align}
    \label{thm:one_source_KL}
    \frac{1}{2}\bigg(\underbrace{\frac{N_0+w_{1}^2n_1}{\left(N_{0}+w_{1}n_1\right)^2}}_{\textnormal{variance~ term}}+\underbrace{\frac{w_{1}^2n_1^2}{(N_{0}+w_{1}n_1)^2}t}_{\textnormal{bias~term}}\bigg)+o\left(\frac{1}{N_{0}}\right),
\end{align}
where 
\begin{align}
t \defeq J\left(\theta_0\right)\left(\theta_1 - \theta_0\right)^2 .
\end{align}
For \textbf{optimal transfer quantity}, by minimizing the above expression, we obtain that 
maximizing the transfer quantity is optimal, 
i.e.,
\begin{align}
\label{optimal_quantity_singlesource}
n_1^* = N_1.
\end{align}

Moreover, the solution of \textbf{optimal transfer weight} $w_1^*$ is
\begin{align}
\label{optimal_weight_singlesource}
w_1^*=\frac{1}{1+tN_1}.
\end{align}

\end{theorem}

In Fig.~\ref{bias_variance_trade_off}, we provide an intuitive explanation for the existence of the optimal weight $w_1^*$ defined in \eqref{optimal_weight_singlesource}.
We now discuss several properties of the optimal weight $w_1^*$. We observe that as the distance between $ \theta_0 $ and $\theta_1$ increases, $t$ increases and $w_1^*$ decreases accordingly. In particular, as $|\theta_0 - \theta_1| \to 0$, $w_1^*$ approaches one, whereas as the discrepancy becomes sufficiently large, $w_1^*$ approaches zero. This shows that the weighting strategy reduces the influence of the source domain when it exhibits substantial discrepancies from the target.
Moreover, as the quantity of source samples $ N_1 $ increases, $w_1^*$ also decreases. This illustrates that the derived weight adaptively controls the trade-off between the target and source information, preventing the source from dominating even when it is abundant.
\begin{figure}[t]
\centering
\includegraphics[width=0.48\textwidth]{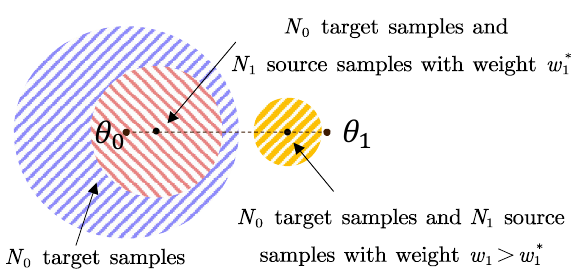}
\caption{The three circles represent the errors corresponding to transfer weight $w_1=0$, $w_1=w_1^*$, and $w_1>w_1^*$. The distance from each circle’s center to $\theta_0$ represents the bias of estimation, and the radius represents the variance. As the transfer weight of source increases, the bias term increases,  while the variance term decreases for $w_1 \in [0,1]$ and increases for $w_1 \in (1, +\infty)$. The optimal $w_1^*$ achieves the best trade-off between them.}
\label{bias_variance_trade_off}
\end{figure}

In the following, we extend the result of Theorem \ref{thm:one_source} to the high-dimensional setting.

\begin{proposition}  
\label{prop:highdimension}(proved in the Appendix)
In the single-source setting with $d$-dimensional models $P_{X;{\underline{\theta}}_0}$ and $P_{X;{\underline{\theta}}_1}$, we assume that ${\underline{\theta}}_0,{\underline{\theta}}_1 \in \mathbb{R}^d$ and $\vert\vert{\underline{\theta}}_0-{\underline{\theta}}_1\vert\vert=O(\frac{1}{\sqrt{N_{0}}})$. Then,
the K-L measure takes the following form:
\begin{align} 
\label{Proposition:one_source_eq_rewrite}
    \frac{d}{2}\left({\frac{N_0+w_{1}^2n_1}{\left(N_{0}+w_{1}n_1\right)^2}}+{\frac{w_{1}^2n_1^2}{(N_{0}+w_{1}n_1)^2}t}\right) + o\left(\frac{1}{N_0}\right),
\end{align}
where
\begin{align}
    t \defeq \frac{({\underline{\theta}}_1 - {\underline{\theta}}_0)^\top J({\underline{\theta}}_0) ({\underline{\theta}}_1 - {\underline{\theta}}_0)}{d}.
\end{align}

Here, $t$ is a scalar, $J({\underline{\theta}}_0) \in \mathbb{R}^{d \times d}$ denotes the Fisher information matrix evaluated at ${\underline{\theta}}_0$, and $({\underline{\theta}}_1 - {\underline{\theta}}_0) \in \mathbb{R}^d$ is the coordinate-wise difference between the two parameter vectors. Finally, the \textbf{optimal transfer quantity} is the same form as \eqref{optimal_quantity_singlesource} and the \textbf{optimal transfer weight} $w_1^*$ is the same form as \eqref{optimal_weight_singlesource}. 
\end{proposition}

Compared with Theorem \ref{thm:one_source}, Proposition \ref{prop:highdimension} retains an analogous structural form in its mathematical expression. This structural resemblance enables us to follow a conceptually similar analytical route to determine the corresponding optimal quantity of knowledge transfer. 
It is worth emphasizing that the appearance of the Fisher information matrix $J$ in the $t$-term is a direct consequence of the theoretical derivation, rather than a modeling choice or an explicitly imposed mechanism. The resulting expression admits a clear interpretation: the Fisher information matrix implicitly induces a dimension-weighted aggregation over parameter directions, where the contribution of each direction is determined by the sensitivity of the likelihood. Parameter directions along which the likelihood is more sensitive to perturbations contribute more prominently to the derived measure, while directions with lower sensitivity have a reduced influence. 


In addition, a closer inspection of Equation \eqref{Proposition:one_source_eq_rewrite} reveals that the associated K-L measure exhibits a growth trend as the parameter dimensionality d increases. From a practical perspective, this scaling trend indicates that as models grow in complexity—characterized by increasingly large parameter spaces—the learning process becomes more intricate, thereby amplifying the challenge of achieving effective knowledge transfer. This observation is consistent with established findings in the literature \cite{tong2021mathematical}, which indicate that increasing model complexity inherently exacerbates the difficulty of cross-task generalization.

\begin{remark} [Connection to Prior Work]\label{remark_Connection to Prior Work}
The expression in Theorem \ref{thm:one_source} and Proposition \ref{prop:highdimension} coincides with the results obtained by \cite{tong2021mathematical}. Equation \eqref{mle_defination} can be transformed to
\begin{align}
    &\hat{\theta}=\argmax_{\theta} \frac{\sum\limits_{x \in X^{N_0}} \log P_{X;\theta}(x) + \sum\limits_{x \in X^{N_1}} w_1 \log P_{X;\theta}(x)}{N_0 + w_1N_1},
    \notag \\&=\argmax_{\theta}\mathbb{E}_{\hat{P}_{\mathrm{mix}}}\left[\log P_{X;\theta}(x)\right],
    \label{remark_eq}
\end{align}
where
\begin{align}
\label{P_MIX}
\hat{P}_{\mathrm{mix}}
:= \frac{N_0}{N_0 + w_1 N_1}\hat{P}(X^{N_0}_0)
 + \frac{w_1 N_1}{N_0 + w_1 N_1}\hat{P}(X^{N_1}_1).
\end{align}
 
Substituting the optimal weight \eqref{optimal_weight_singlesource} into \eqref{P_MIX}, we have
\begin{align}
\hat{P}_{\mathrm{mix}}
:= \frac{t+\frac{1}{N_1}}{t+\frac{1}{N_0}+\frac{1}{N_1}}\hat{P}(X^N_0)+\frac{\frac{1}{N_0}}{t+\frac{1}{N_0}+\frac{1}{N_1}}\hat{P}(X^N_1).
\end{align}

This formulation exhibits a notable similarity to the conclusions 
in~\cite{tong2021mathematical}. In particular, both frameworks 
contain terms that scale with the inverse of the transfer 
quantities, highlighting the importance of accounting for sample 
size when determining the transfer weights.
However, our framework differs from their work not only in the 
weighting mechanism but also in the treatment of model discrepancy 
and dimensionality. Specifically, Tong et al.~\cite{tong2021mathematical} 
adopt a model-based weighting strategy, whereas our approach employs 
a sample-based weighting scheme. Their analysis quantifies task 
similarity via a $\chi^2$-distance, while the discrepancy term $t$ 
in our formulation measures model distance directly in the parameter 
space and incorporates model complexity through Fisher information. 
Moreover, parameter dimensionality enters the weighting mechanism 
differently in the two approaches.


\end{remark}



\subsection{Multi-Source Transfer Learning
}
\label{subsection_multi-Source Transfer Learning}
In this subsection, we first present the expression of the K-L measure in the multi-source setting, as stated in Theorem \ref{thm:Weighted_multi_source}. By minimizing this expression, we derive a method to compute the optimal source weights and analyze its computational complexity. Moreover, we provide an interpretation of the theoretical result that using all available samples is optimal when the weights are adjustable.


\begin{theorem}
\label{thm:Weighted_multi_source}(proved in the Appendix)
In the multi-source setting with $d$-dimensional models $P_{X;\underline{\theta}_0}$ for the target task and $P_{X;\theta_i}$, $i\in[1,K]$ for the source tasks, we assume that $\underline{\theta}_0,\underline{\theta}_1 \dots \underline{\theta}_K\in \mathbb{R}^d$ and $\vert\vert\underline{\theta}_0-\underline{\theta}_i\vert\vert=O(\frac{1}{\sqrt{N_{0}}})$. Then, the K-L~measure is given by 
\begin{align}
\label{eq:Weighted_multisourcetarget}
\frac{d}{2}\left(\frac{N_0}{(N_0+{s})^2}+\frac{{s}^2}{(N_0+{s})^2}t\right)+ o\left(\frac{1}{N_0}\right),
\end{align}
where we denote ${{b}_i}=n_iw_i$, ${s}=\sum\limits_{i=1}^{K}{{b}_i}$, ${{\alpha}_i}=\frac{{{b}_i}}{{s}}$, and 
\begin{align}
\label{t_multi_source}
t=\frac{\underline{\alpha}^T\left(\operatorname{diag}\left(\frac{d}{n_1}, \dots, \frac{d}{n_K}\right)
+\Theta^T{J}(\underline{\theta}_0)\Theta\right)\underline{\alpha}}{d}.
\end{align}
In \eqref{t_multi_source}, $\underline{\alpha}=\left[\alpha_1,\dots,\alpha_K\right]^T$ is a K-dimensional vector, and 
\begin{align}
{{\Theta}}^{d\times K}=\left[{{\underline{\theta}}_1}-{{\underline{\theta}}_0},\dots,{{\underline{\theta}}_K}-{{\underline{\theta}}_0} \right].
\end{align}
For \textbf{optimal transfer quantities}, by minimizing \eqref{eq:Weighted_multisourcetarget}, we obtain that 
maximizing the transfer quantities is optimal, i.e., $n_i^{*} = N_i$.
Moreover, we propose a method to compute the \textbf{optimal source weights}, which will be elaborated in the following paragraph.
\end{theorem}

We next present the method to compute the optimal value of $w_i$ to minimize \eqref{eq:Weighted_multisourcetarget}. The minimization problem is
\begin{align}
(s^*, \underline{\alpha}^*) \gets \argmin\limits_{(s, \underline{\alpha})} \frac{d}{2}\Bigg(\frac{N_0}{(N_0+{s})^2}+\notag\\
\frac{{s}^2}{(N_0+{s})^2}t\Bigg)+ o\left(\frac{1}{N_0}\right).
\end{align}

We reformulate this optimization problem as a sequential optimization process, and explicitly formulate the constraints as follows.
\begin{align}
&(s^*, \underline{\alpha}^*) \gets \argmin\limits_{s\in[0,+\infty]} \frac{d}{2}\Bigg(\frac{N_0}{(N_0+{s})^2}+\frac{s^2}{(N_0+s)^{2}}\cdot\\&\argmin\limits_{\underline{\alpha}\in\cA}\frac{\underline{\alpha}^T\left(\operatorname{diag}\left(\frac{d}{N_1}, \dots, \frac{d}{N_K}\right)  
+\Theta^T{J}(\underline{\theta}_0)\Theta\right)\underline{\alpha}}{d}\Bigg),
\end{align}
where 
\begin{align}
\cA=\left\{\underline{\alpha}\Bigg|\sum_{i=1}^{K}\alpha_i=1,  \alpha_i \ge0,i=1,\dots,K\right\}.
\end{align}

This problem requires optimizing the objective function over two variables: a scalar variable $s$ representing the total weighted transfer quantity, and a vector variable $\underline{\alpha}$ representing the proportion of each source domain in $s$. First, we compute the optimal $\underline{\alpha}^*=\left[\alpha_1^*,\dots,\alpha_K^*\right]^T$ under the constraint $\mathcal{A}$ to minimize $t$, which is a $K\times K$ non-negative quadratic programming problem with respect to $\underline{\alpha}$, i.e.,

\begin{align}
\label{eq_definition_optimal_alpha}
\underline{\alpha}^*=\arg\min\limits_{\underline{\alpha}\in\mathcal{A}}\underline{\alpha}^T\frac{\left(\operatorname{diag}\left(\frac{d}{N_1}, \dots, \frac{d}{N_K}\right)
+\Theta^T{J}(\underline{\theta}_0)\Theta\right)}{d}\underline{\alpha}. 
\end{align}
In the quadratic coefficient matrix of this optimization problem, the first component, $\operatorname{diag}\left(\frac{d}{N_1}, \dots, \frac{d}{N_K}\right)$, is diagonal and therefore strictly positive definite. The second component, $\Theta^\top J(\underline{\theta}_0)\Theta$, is positive semi-definite since the Fisher information matrix $J(\underline{\theta}_0)$ is positive definite by the assumptions in Section~\ref{section_Preliminaries}. Consequently, the sum of these two components is positive definite, which guarantees that the optimization problem is convex and admits a unique global optimum.
By this procedure, we obtain $t^*$ and $\alpha^*$ that minimize $t$. Then we get the optimal value of ${s}$, which is always \begin{align}{s^*}=\frac{1}{t^*}.\end{align} Finally, we use optimal ${s^*}$ and ${\alpha}^*_i$ to get the optimal weights by   
\begin{align}\label{eq_definition_optimal_weight} w^*_i=\frac{s^*{\alpha}^*_i}{N_i}. \end{align}

The time complexity of this optimization is dominated by solving a $K \times K$ quadratic programming problem over $\alpha$, which can be solved in $\mathcal{O}(K^3)$ time using a primal-dual interior-point solver.
In most standard transfer learning benchmarks, the number of available source domains is modest (typically $K \leq 10$), rendering the $\mathcal{O}(K^3)$ computational complexity manageable for real-world applications. 
For cases where the number of available source domains is substantially larger, one feasible strategy to preserve scalability is to perform a preliminary clustering or aggregation of domains with high similarity before executing our algorithm.

Moreover, the finding that maximizing the transfer quantity is optimal is non-trivial and theoretically insightful. Conventional wisdom in transfer learning holds that excessive use of source data without controlling domain discrepancy can amplify bias and cause negative transfer \cite{zhang2022survey}. This perspective has motivated many approaches that either down-weight or selectively discard source samples to mitigate distribution mismatch \cite{jain2023data,li2024scalable, zhang2025theoretical}. On the other hand, some methods by default exploit all source data for joint training, often relying on heuristic weighting. Our framework provides a new theoretical perspective by jointly optimizing transfer quantity and source weights. Through asymptotic analysis, we prove that when weights are allowed to adjust, the optimal solution always employs all available source samples. The key insight is that enlarging the effective sample size strictly reduces variance, while the weighting mechanism simultaneously suppresses the bias introduced by heterogeneous domains. 
This dual effect reconciles prior intuitions: discarding data may reduce bias but weakens variance reduction, while our analysis shows that full data utilization with proper weighting achieves a better bias–variance trade-off.
Consequently, our analysis offers a principled justification for exploiting all source data in multi-source transfer learning.


\begin{figure}[htbp]
\centering
\includegraphics[width=0.45\textwidth]{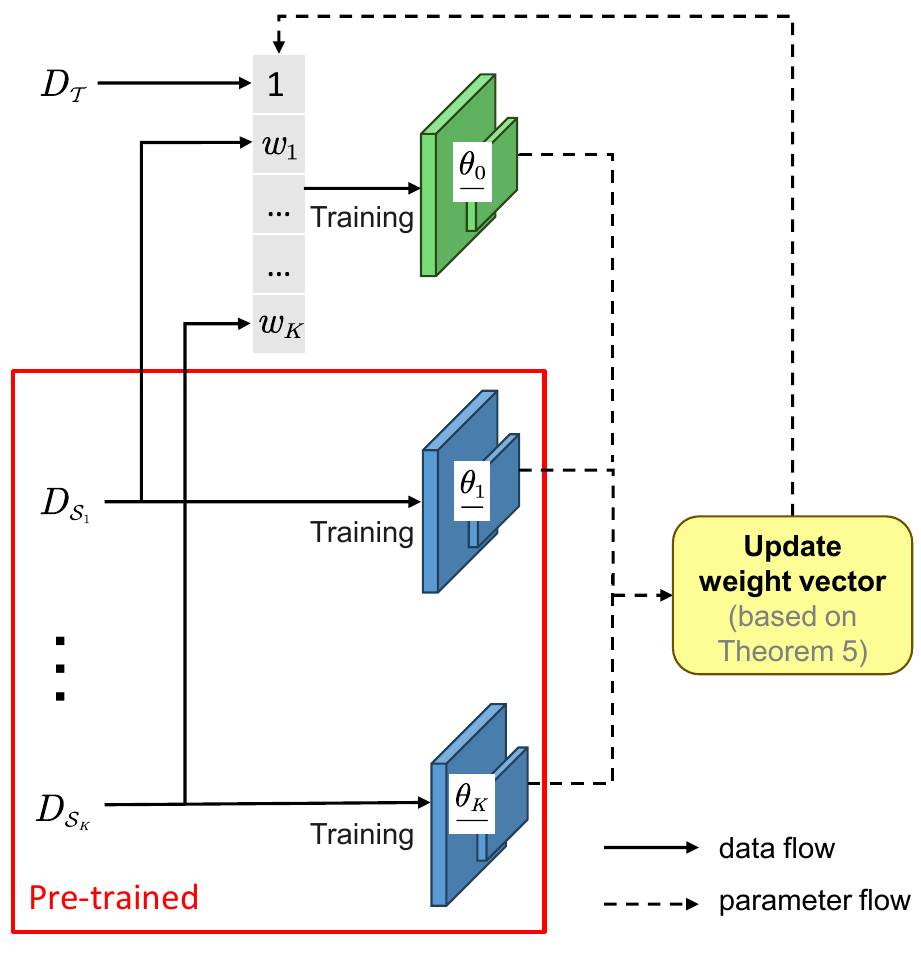}
\caption{\textbf{Overview of the \ourmethod{} training pipeline}: At each iteration, the model parameters are optimized via gradient descent using target data together with weighted samples from each source task. The source weights are subsequently updated based on the current model parameters. This alternating optimization procedure establishes an iterative feedback loop that jointly updates the source weights and the target model.}
\label{overview}
\end{figure}

\section{Practical Algorithms}
\label{section_Practical_Algorithms}
In this section, we introduce practical algorithms derived from our theory for multi-source transfer learning and multi-task learning.
Specifically, we propose practical algorithms of \ourmethod{}
for multi-source transfer learning  
in Algorithm \ref{alg:ours_algorithm} and for multi-task learning in Algorithm \ref{alg:multi_task_alternating}.

Firstly, we focus on Algorithm~\ref{alg:ours_algorithm} to illustrate the proposed approach, whose overall workflow is shown in Fig.~\ref{overview}.
In \ourmethod{}, when computing the optimal source weights based on Theorem \ref{thm:Weighted_multi_source}, we use the parameters of pretrained source models to replace $\underline{\theta}_1,\cdots,\underline{\theta}_K$. Considering that the source model can be trained using sufficient labeled data, it is reasonable to use the learned parameters as a good approximation of the true underlying parameters.
In contrast, the amount of target data in transfer learning is often insufficient, so it is difficult to accurately estimate the true parameter $\underline{\theta}_0$ - the parameter of the target task model - using only the target data. 
Therefore, as shown in lines 5-12 of Algorithm 1, we adopt a \textbf{dynamic weight update strategy}. 
Specifically, in the first epoch, we train $\underline{\theta}_0$ with the source weights initialized to zero, which is equivalent to using only the target data. This $\underline{\theta}_0$ is then used, along with Theorem~\ref{thm:Weighted_multi_source}, to determine the optimal source weight for each source task.
Thereafter, training proceeds in an alternating manner: given the current weights, we update $\underline{\theta}_0$, and given the updated model, we recompute the weights. This process is repeated at each epoch, enabling iterative refinement of both the model and the weights. Such an iterative scheme allows adaptive weight adjustment, resulting in more stable optimization and improved transfer performance.
In particular, the loss function $\ell$ is the negative log-likelihood. Moreover, in line 9, we compute the matrix $J$ using the widely adopted \textbf{empirical Fisher} approach in deep learning \cite{martens2020new,osawa2023pipefisher}.

We further extend the proposed framework to the multi-task 
learning setting, where each task simultaneously serves as both a target and a source for the others, and the objective is to learn task-specific models for all tasks.
Suppose we are given $K$ tasks $\{S_k\}_{k=1}^K$, 
where each task $S_k$ is associated with dataset $D_{S_k}$ 
and parameter $\theta_k$. 
When optimizing task $k$, we treat $S_k$ as the target task 
and the remaining tasks $\{S_{k'}\}_{k' \ne k}$ as auxiliary sources, and we construct its training set by combining all samples from the target task with the weighted samples from the remaining tasks, where the weights $w_{k'}^{(k)}$ are computed by \ourmethod{}.
The resulting problem becomes a joint optimization over all task parameters 
$\{\theta_k\}_{k=1}^K$, leading to an alternating optimization procedure between tasks, 
summarized in Algorithm~\ref{alg:multi_task_alternating}.

\section{Experiments}
\label{section_Experiments}

\subsection{Experimental Settings}
In the experiments, we evaluated the practical algorithms of \ourmethod{} in real-world datasets. In particular, we 
evaluated Algorithm~\ref{alg:ours_algorithm} in the multi-source transfer learning setting and Algorithm~\ref{alg:multi_task_alternating} in the multi-task learning setting. 

\textbf{Benchmark Datasets.} 
The \texttt{Office-Home} dataset \cite{venkateswara2017deep} consists of four domains: Art (\textbf{Ar}, 2,427 samples), Clipart (\textbf{Cl}, 4,365 samples), Product (\textbf{Pr}, 4,439 samples), and Real World (\textbf{Rw}, 4,357 samples). Each domain contains images from the same 65 object categories, resulting in a total of 15,588 samples. The \texttt{DomainNet} dataset \cite{peng2019moment_M3SDA} comprises six distinct domains: Clipart (\textbf{Cl}, 48,841 samples), Infograph (\textbf{In}, 48,466 samples), Painting (\textbf{Pa}, 48,529 samples), Quickdraw (\textbf{Qd}, 48,755 samples), Real (\textbf{Re}, 120,906 samples), and Sketch (\textbf{Sk}, 49,044 samples). Each domain includes images from the same 345 object categories, yielding a total of 364,541 samples. Data examples of the two datasets are shown in Fig.~\ref{fig:dataset}.
We treat the multi-class classification problem on each domain as an individual task. 
\begin{figure}[htbp]
\centering
\includegraphics[width=0.48\textwidth]{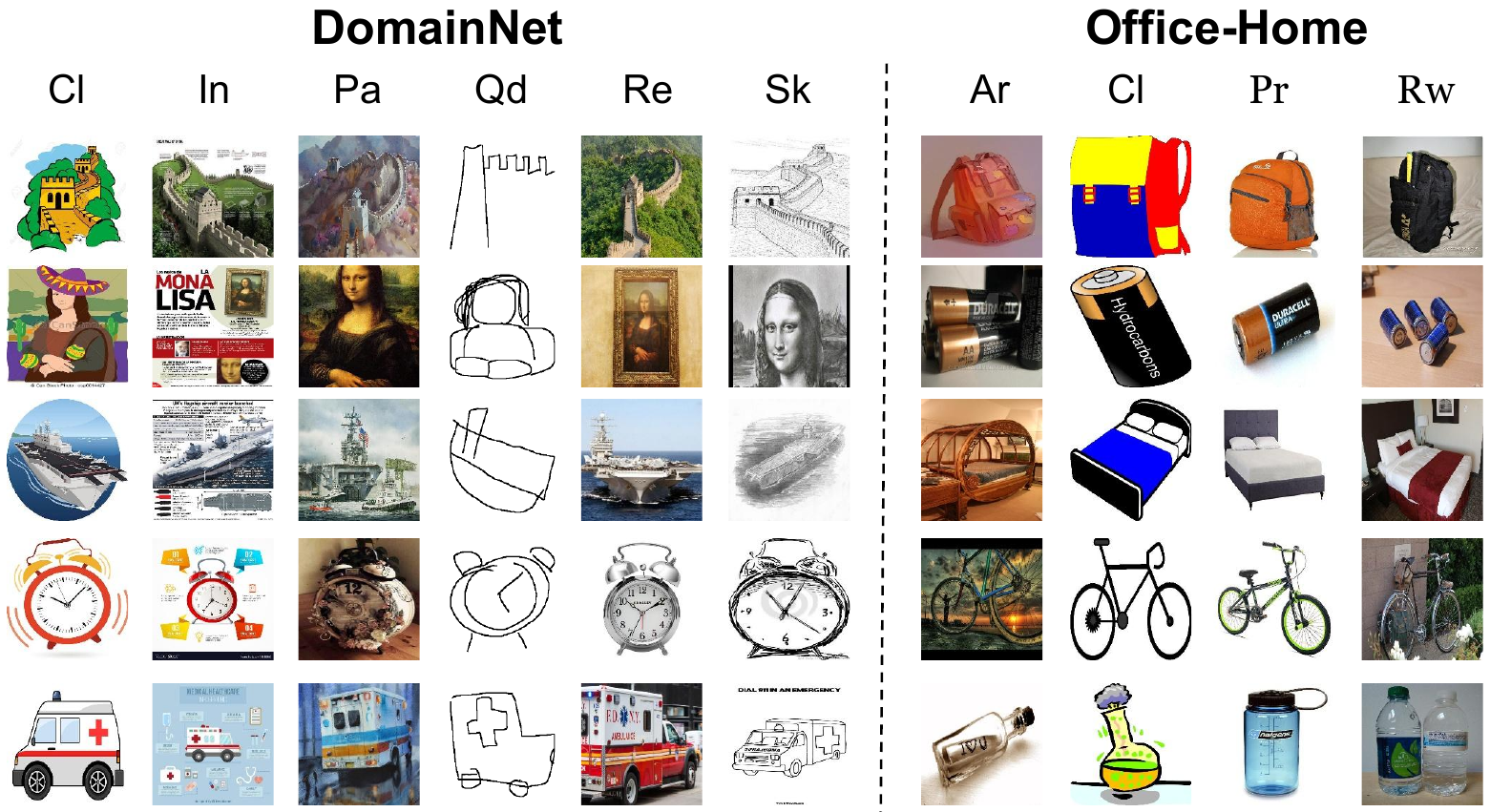}
    \caption{\textbf{Overview of the Datasets}: Examples from the cross-domain datasets \texttt{DomainNet} and \texttt{Office-Home}, where images from different domains exhibit distinct visual styles or are captured using different devices.
}
\label{fig:dataset}
\end{figure}


\begin{algorithm}[H]
   \caption{\ourmethod{} for Multi-Source Transfer Learning}
   \label{alg:ours_algorithm}
\begin{algorithmic}[1]
    \STATE {\bfseries Input:} Target data $D_{\cT}=\{(z_\cT^j, y_\cT^j)\}_{j=1}^{N_0}$, source data $\{D_{S_k}=\{(z_{S_k}^j, y_{S_k}^j)\}_{j=1}^{N_{k}}\}_{k=1}^K$, model type $f_{{\underline{\theta}}}$ and its parameters ${\underline{\theta}}_0$ for target task and $\{{\underline{\theta}}_{k}\}_{k=1}^K$ for source tasks, parameter dimension $d$, source weights $\{w_k\}_{k=1}^K$\\
    \mycomment{$z$ represents the feature and $y$ represents the label}
    \STATE {\bfseries Parameter:} Learning rate $\eta$.  
   
    \STATE {\bfseries Initialize:} randomly initialize ${\underline{\theta}}_0$, use parameters of pretrained source models to initialize $\{{\underline{\theta}}_{k}\}_{k=1}^K$, initialize weights $\{w_k\}_{k=1}^K=0$
   
    \STATE {\bfseries Output:} a well-trained $\underline{\theta}_0$ for target task model $f_{{\underline{\theta}}_0}$

   
    \STATE \textbf{repeat} 
    \STATE \hspace{1em}$\cL_{train} \gets\frac{\sum\limits_{(z^j,y^j)\in{D_\cT}} \ell\left(y^j, f_{{\underline{\theta}}_0}(z^j)\right)}{| D_\cT \bigcup \{D_{S_k}\}_{k=1}^K|}$\\
    
    $\qquad\qquad+\frac{\sum\limits_{k=1}^{K}\sum\limits_{(z^j,y^j)\in{D_{S_k}}} \ell\left(y^j, f_{{\underline{\theta}}_0}(z^j)\right)*w_k}{| D_\cT \bigcup \{D_{S_k}\}_{k=1}^K|}$
    
    \STATE \hspace{1em}${\underline{\theta}}_0 \gets {\underline{\theta}}_0 - \eta \nabla_{{\underline{\theta}}_0} \cL_{train}$
   
    \STATE \hspace{1em}$\Theta\gets\left[{{\underline{\theta}}_1}-{{\underline{\theta}}_0},\dots,{{\underline{\theta}}_K}-{{\underline{\theta}}_0} \right]^T$
    

    \STATE \hspace{1em}${J}({{\underline{\theta}}_0}) \gets \frac{\sum\limits_{(z^j,y^j)\in{D_\cT}}\nabla_{{\underline{\theta}}_0} \ell\left(y^j, f_{{\underline{\theta}}_0}(z^j)\right) \left( \nabla_{\underline{\theta}_0} \, \ell\!\left(y^j, f_{\underline{\theta}_0}(z^j)\right) \right)^{\top}
}{| D_\cT |}$ 

    \STATE \hspace{1em}$(s^*, \underline{\alpha}^*) \gets \argmin\limits_{(s, \underline{\alpha})} \frac{d}{2}\left(\frac{N_0}{(N_0+{s})^2}+\frac{{s}^2}{(N_0+{s})^2}t\right)$ 
    \STATE \hspace{1em} $w^*_k \gets\frac{s^*{\alpha}^*_k}{N_k},k=1,\dots,K$
    \STATE \textbf{until} ${\underline{\theta}}_{0}$ converges;
\end{algorithmic}
\end{algorithm}

\begin{algorithm}[H]
   \caption{\ourmethod{} for Multi-Task Learning}
   \label{alg:multi_task_alternating}
\begin{algorithmic}[1]
    \STATE {\bfseries Input:} Datasets $\{D_{S_k}=\{(z_{S_k}^j, y_{S_k}^j)\}_{j=1}^{N_k}\}_{k=1}^K$, model type $f_{\underline{\theta}}$, parameters $\{\underline{\theta}_k\}_{k=1}^K$, parameter dimension $d$, weights $\{w_{k'}^{(k)}\}_{k,k'=1,\;k'\neq k}^K$
    \STATE {\bfseries Parameter:} Learning rate $\eta$
    \STATE {\bfseries Initialize:} Randomly initialize $\{\underline{\theta}_k\}_{k=1}^K$, set weights $w_{k'}^{(k)} = 0$ for all $k'\neq k$
    \STATE {\bfseries Output:} well-trained parameters $\{\underline{\theta}_k\}_{k=1}^K$
    \REPEAT
        \FOR{each $k = 1$ to $K$}
            \STATE \quad$\cL_{train}^{(k)} \gets\frac{
           \sum\limits_{(z^j,y^j)\in{D_{S_k}}} \ell\left(y^j, f_{\underline{\theta}_k}(z^j)\right)
            }{| \{D_{S_k}\}_{k=1}^K|}$\\
           \qquad\qquad$+\frac{ \sum\limits_{k' \ne k} \sum\limits_{(z^j,y^j)\in{D_{S_{k'}}}} \ell\left(y^j, f_{\underline{\theta}_k}(z^j)\right)\cdot w_{k'}^{(k)}}{|  \{D_{S_k}\}_{k=1}^K |}$
            
            \STATE \quad$\underline{\theta}_k \gets \underline{\theta}_k - \eta \nabla_{\underline{\theta}_k} \cL_{train}^{(k)}$
            \STATE \quad$\Theta^{(k)} \gets \bigg[\underline{\theta}_1 - \underline{\theta}_k, \dots, \underline{\theta}_{k-1} - \underline{\theta}_k,$\\\qquad\qquad$
            \underline{\theta}_{k+1} - \underline{\theta}_k, \dots, \underline{\theta}_K - \underline{\theta}_k \bigg]$
            \STATE \quad$J^{(k)} \gets \frac{\sum\limits_{\substack{(z^j,y^j) \\ \in D_{S_k}}}\nabla_{\underline{\theta}_k} \ell\left(y^j, f_{\underline{\theta}_k}(z^j)\right) \left( \nabla_{\underline{\theta}_k} \, \ell\!\left(y^j, f_{\underline{\theta}_k}(z^j)\right) \right)^{\top}  
            }{| D_{S_k} |}$
            \STATE \quad$(s^*_k, \underline{\alpha}^{*(k)}) \gets \argmin\limits_{(s, \underline{\alpha})} \frac{d}{2}\left(\frac{N_k}{(N_k+s)^2} + \frac{s^2}{(N_k+s)^2}t\right)$
            \STATE \quad$w_{k'}^{(k)} \gets \frac{s^*_k \alpha_{k'}^{*(k)}}{N_{k'}}, \quad \forall k' \ne k$
        \ENDFOR
    \UNTIL{all $\underline{\theta}_k$ converge}
\end{algorithmic}
\end{algorithm}


In the multi-source transfer learning setting, our goal is to transfer useful knowledge from multiple source domains to a specific target domain.   
In the experiments under this setting, our evaluation is performed in the \textbf{supervised-10-shot} learning setting, where only 10 labeled samples per class are available in the target domain, while all available samples from the source domains in the dataset are utilized.

In the multi-task learning setting, each task simultaneously serves as both a target task to be learned and a source task from which knowledge is transferred to other tasks. In the experiments under this setting, all available data from every domain in the dataset are fully utilized.


\textbf{Implementation Details.}
Experiments are conducted on  NVIDIA A800 GPU.
To ensure fair comparison with prior work, different configurations are adopted for two scenarios.
For the multi-source transfer learning setting, we adopt ViT-Small \cite{rw2019timm_vits} pre-trained on ImageNet-21k \cite{deng2009imagenet} as the backbone and Adam optimizer \cite{kingma2014adam} with a learning rate of $10^{-5}$. We allocate 20\% of the dataset as the test set and report the best accuracy within 5 epochs of early stopping.
For the multi-task learning setting, we use ResNet-18 \cite{he2016deep_resnet} pre-trained on ImageNet-1k \cite{krizhevsky2012imagenet} as the backbone and Adam optimizer with a learning rate of $10^{-4}$. The dataset is randomly divided into 60\% for training, 20\% for validation, and 20\% for testing.

\begin{table*}[!htbp]
\centering
\caption{\textbf{Multi-Source Transfer Performance on DomainNet and Office-Home.} The arrows indicate transferring from the rest tasks. The highest/second-highest accuracy is marked in Bold/Underscore form respectively. }
\resizebox{\textwidth}{!}{
\begin{tabular}{lc c ccccccc c ccccc}
\toprule
\multirow{2.5}{*}{\textbf{Method}} & \multirow{2.5}{*}{\textbf{Backbone}} && \multicolumn{7}{c}{\textbf{DomainNet}} && \multicolumn{5}{c}{\textbf{Office-Home}}\\ 
\cmidrule(lr){4-10} \cmidrule(lr){12-16}
 &&&$\to$C & $\to$I & $\to$P & $\to$Q & $\to$R & $\to$S & Avg&& $\to$Ar & $\to$Cl & $\to$Pr & $\to$Rw & Avg\\ 

\midrule
\multicolumn{14}{l}{\textbf{\textrm{Unsupervised-all-shots}}} \\
MSFDA\cite{shen2023balancingMSFDA} & ResNet50 && 66.5 & 21.6 &56.7 &20.4 &70.5 &54.4 &48.4  && 75.6 &62.8 &84.8 &85.3 &77.1  \\
DATE\cite{han2023discriminability_DATE}& ResNet50 && - & - & - & - & - & -   & - && 75.2 & 60.9 &  \underline{85.2} & 84.0 & 76.3 \\
M3SDA\cite{peng2019moment_M3SDA}& ResNet101 &&57.2 & 24.2 & 51.6 & 5.2 & 61.6& 49.6& 41.5  && - & - & - & - & - \\
\midrule
\multicolumn{14}{l}{\textbf{\textrm{Supervised-10-shots}}} \\

\multicolumn{14}{l}{\textbf{\textit{Source-Ablation Methods:}}} \\
Target-Only& ViT-S && 14.2 & 3.3  & 23.2  & 7.2  & 41.4   & 10.6  & 16.7 && 40.0 & 33.3 & 54.9 & 52.6 & 45.2 \\ 

Single-Source-Avg& ViT-S && 50.4 & 22.1 & 44.9  & 24.7  & 58.8 & 42.5  & 40.6 && 65.2  & 53.3 & 74.4 & 72.7 & 66.4 \\
Single-Source-Best& ViT-S && 60.2 & 28.0 & 55.4  & 28.4 & 66.0 & 49.7 & 48.0 && 72.9  & 60.9 & 80.7 & 74.8 & 72.3 \\ 

\allsource{}& ViT-S && 71.7 & 32.4 & 60.0 & 31.4 & 71.7 & 58.5 & 54.3 && 77.0 & 62.3 & 84.9 & 84.5 & 77.2\\

\multicolumn{14}{l}{\textit{\textbf{Model-Weighting Based Few-Shot Methods:}}}\\
MCW\cite{lee2019learning_MCW}& ViT-S && 54.9 & 21.0 & 53.6 & 20.4 & 70.8 & 42.4 & 43.9 && 68.9 & 48.0 & 77.4 & \underline{86.0} & 70.1 \\
H-ensemble\cite{wu2024h_Hensemble}& ViT-S && 53.4 & 21.3 & 54.4 & 19.0 & 70.4 & 44.0 & 43.8 && 71.8 & 47.5 & 77.6 & 79.1 & 69.0  \\



\multicolumn{14}{l}{\textit{\textbf{Sample-Based Few-Shot Methods:}}}\\
WADN\cite{shui2021aggregating_WADN}& ViT-S && 68.0 & 29.7 & 59.1 & 16.8 & \underline{74.2} & 55.1 & 50.5 && 60.3 & 39.7 & 66.2 & 68.7 & 58.7 \\
MADA\cite{zhang2024revisiting_MADA}& ViT-S && 51.0 & 12.8 & 60.3 & 15.0 & \textbf{81.4} & 22.7 & 40.5 && \underline{78.4} & 58.3 & 82.3 & 85.2 & 76.1 \\
OTQMS & ViT-S && \underline{72.8} & \underline{33.8} & \underline{61.2} & \underline{33.8} & 73.2 & \underline{59.8} & \underline{55.8} && 78.1 & \underline{64.5} & \underline{85.2} & 84.9  & \underline{78.2} \\

\ourmethod{} & ViT-S && \textbf{74.5} & \textbf{35.2} & \textbf{61.8} & \textbf{36.8}  & 73.3 & \textbf{60.9} & \textbf{57.1} && \textbf{78.7} & \textbf{65.1} & \textbf{86.9} & \textbf{87.6} & \textbf{79.6} \\
\bottomrule
\end{tabular}
}
\label{tab:major}
\end{table*}

\begin{table*}[!htbp]
\centering
\caption{\textbf{Multi-Task Learning Performance on DomainNet and Office-Home.} 
Each column indicates the target domain used for evaluation. 
The highest and second-highest accuracy are marked in \textbf{bold} and \underline{underline}, respectively.}
\resizebox{\textwidth}{!}{
\begin{tabular}{l ccccccc c ccccc}
\toprule
\multirow{2.5}{*}{\textbf{Method}} 
& \multicolumn{7}{c}{\textbf{DomainNet}} 
& & \multicolumn{5}{c}{\textbf{Office-Home}} \\
\cmidrule(lr){2-8} \cmidrule(lr){10-14}
& C & I & P & Q & R & S & Avg && Ar & Cl & Pr & Rw & Avg \\
\midrule
EW 
& 67.3 & \textbf{22.7} & \underline{56.1} & 35.1 & 69.3 & 55.9 & 51.1
&& 63.0 & 76.5 & 88.5 & 77.7 & 76.4 \\
MGDA-UB\cite{sener2018multi} 
& \textbf{69.0} & 22.4 & \underline{56.1} & 31.7 & 69.7 & 55.9 & 50.8
&& 64.3 & 75.3 & \underline{89.7} & 79.3 & 77.2 \\
GradNorm\cite{chen2018gradnorm} 
& 67.1 & \underline{22.5} & \underline{56.1} & 35.7 & 69.3 & 55.8 & 51.1
&& 65.5 & 75.3 & 88.7 & 78.9 & 77.1 \\
PCGrad\cite{yu2020gradient} 
& 67.8 & \textbf{22.7} & 56.0 & 34.7 & 69.2 & 56.1 & 51.1
&& 63.9 & 76.0 & 88.9 & 78.3 & 76.8 \\
CAGrad\cite{liu2021conflict} 
& \underline{68.4} & \textbf{22.7} & 55.7 & 32.1 & \underline{69.8} & 55.6 & 50.7
&& 63.8 & 75.9 & 89.1 & 78.3 & 76.8 \\
RGW\cite{lin2022reasonable_RGW} 
& 67.1 & 21.9 & 55.2 & 33.9 & 68.6 & 55.6 & 50.4
&& 65.1 & \underline{78.7} & 88.7 & 79.9 & 78.1 \\
MoCo\cite{fernando2022mitigating} 
& 60.3 & 19.0 & 46.6 & \textbf{39.9} & 57.9 & 50.7 & 45.7
&& 64.1 & \textbf{79.8} & 89.6 & 79.6 & 78.3 \\
MoDo\cite{chen2023three} 
& 67.9 & 22.0 & 55.4 & 36.0 & 69.4 & \underline{56.3} & \underline{51.2}
&& \underline{66.2} & 78.2 & \textbf{89.8} & \underline{80.3} & \underline{78.7} \\
\ourmethod{} 
& 68.2 & \textbf{22.7} & \textbf{56.5} & \underline{37.2} & \textbf{70.0} & \textbf{57.2} & \textbf{51.9}
&& \textbf{69.3} & 77.3 & 89.5 & \textbf{80.4} & \textbf{79.1} \\
\bottomrule
\end{tabular}
}
\label{tab:multitask_combined}
\end{table*}

\begin{table*}[!h]
\centering
\caption{Training time comparison on the DomainNet dataset across different target domains.
}
\resizebox{\textwidth}{!}{
\begin{tabular}{l c ccccccc}
\toprule
\multirow{2.5}{*}{\textbf{Method}} & \multirow{2.5}{*}{\textbf{Backbone}} & \multicolumn{7}{c}{\textbf{DomainNet}}\\ 
\cmidrule(lr){3-9} 
&&$\to$C & $\to$I & $\to$P & $\to$Q & $\to$R & $\to$S & Avg\\
\midrule
MADA      & ViT-S
& 15:57:00 & 20:01:00 & 17:23:30 & 24:07:30 & 11:15:00 & 14:03:30 & 17:07:55 \\
\allsource{} & ViT-S
& 10:34:30 & 10:40:00 & 10:11:00 & 08:35:00 & 08:37:30 & 10:23:00 & 09:50:10 \\
\ourmethod{}      & ViT-S 
& 10:40:54 & 10:50:00 & 10:19:50 & 08:48:18 & 08:52:36 & 10:31:38 & 10:00:32 \\
\bottomrule
\end{tabular}
}
\label{tab:time}
\end{table*}

\begin{table}[!h]
\centering
\caption{Comparison between static and dynamic weighting strategies under the 10-shot setting on the Office-Home dataset.
}
\resizebox{0.48\textwidth}{!}{
\begin{tabular}{l c ccccc}
\toprule
\multirow{2.5}{*}{\textbf{Method}} & \multirow{2.5}{*}{\textbf{Backbone}} & \multicolumn{5}{c}{\textbf{Office-Home}}\\ 
\cmidrule(lr){3-7} 
&& $\to$Ar & $\to$Cl & $\to$Pr & $\to$Rw & Avg\\ \midrule
\multicolumn{7}{l}{\textit{\textbf{Supervised-10-shots:}}} \\
Static & ViT-S & 74.8 & 63.1 & 81.6 & 83.2 & 75.7 \\
Dynamic & ViT-S & \textbf{78.7} & \textbf{65.1} & \textbf{86.9} & \textbf{87.6}  & \textbf{79.6} \\
\bottomrule
\end{tabular}
}
\label{tab:dynamic}
\end{table}

\begin{table}[!h]
\centering
\caption{Multi-Source Transfer with LoRA on Office-Home. We apply LoRA on ViT-B backbone for PEFT.}
\resizebox{0.48\textwidth}{!}{
\begin{tabular}{l c ccccc}
\toprule
\multirow{2.5}{*}{\textbf{Method}} & \multirow{2.5}{*}{\textbf{Backbone}} & \multicolumn{5}{c}{\textbf{Office-Home}}\\ 
\cmidrule(lr){3-7} 
&& $\to$Ar & $\to$Cl & $\to$Pr & $\to$Rw & Avg\\ \midrule
\multicolumn{7}{l}{\textit{\textbf{Supervised-10-shots Source-Ablation:}}} \\
Target-Only & ViT-B & 59.8 & 42.2 & 69.5 & 72.0 & 60.9 \\ 
Single-Source-avg & ViT-B & 72.2  & 59.9 & 82.6 & 81.0 & 73.9 \\
Single-Source-best & ViT-B & 74.4  & 61.8 & 84.9 & 81.9 & 75.8 \\ 
\allsource{} & ViT-B & \underline{81.1} & 66.0 & 88.0 & 89.2 & 81.1\\
OTQMS & ViT-B & \textbf{81.5} & \textbf{68.0} & \underline{89.2} & \textbf{90.3}  & \underline{82.3} \\
\ourmethod{} & ViT-B & \textbf{81.5} & \underline{67.6} & \textbf{90.1} & \underline{90.2}  & \textbf{82.4} \\
\bottomrule
\end{tabular}
}
\label{tab:osw_lora}
\vspace{-2em}  
\end{table}

\subsection{Experimental Design and Model Adaptation}
\label{appendix:Experimental_Design_and_Model_Adaptation}

\textbf{Baselines.}
We compare our method with several SOTA methods in both settings.

For the multi-source transfer learning setting, we follow a unified evaluation setup for all methods. The baselines include:
1) Source Ablation Studies: Target-Only, Single-Source-Avg/Single-Source-Best (average/best single-source transfer results), \allsource{} (transfer results using all available source data without weighting).  
2) Model-Weighting Based Few-Shot Methods: H-ensemble \cite{wu2024h_Hensemble}, MCW \cite{lee2019learning_MCW}.
3) Sample-Based Few-Shot Methods: MADA \cite{zhang2024revisiting_MADA},  WADN~\cite{shui2021aggregating_WADN}, OTQMS~\cite{zhang2025theoretical}.
Furthermore, we report the results of several unsupervised methods for reference, including MSFDA \cite{shen2023balancingMSFDA}, DATE \cite{han2023discriminability_DATE}, M3SDA \cite{peng2019moment_M3SDA}. For these unsupervised methods, we adopt a different experimental setting: instead of following the supervised-10-shot setting, we utilize all available unlabeled target data, corresponding to an unsupervised-all-shot setting.


For the multi-task learning setting, we follow the same experimental setup as in~\cite{chen2023three} and compare with the baselines reported in their work. The baselines include:
Equal Weighting (\textbf{EW}),  
MGDA-UB~\cite{sener2018multi},  
GradNorm~\cite{chen2018gradnorm},  
PCGrad~\cite{yu2020gradient},  
CAGrad~\cite{liu2021conflict},  
RGW~\cite{lin2022reasonable_RGW},  
MoCo~\cite{fernando2022mitigating},  
and MoDo~\cite{chen2023three}.

\subsection{Few-Shot Multi-Source Transfer Performance
}

We evaluated our method, \ourmethod{}, against the baseline methods using the \texttt{Office-Home} and the \texttt{DomainNet} datasets. 
The quantitative results for the multi-source transfer learning setting are summarized in Table~\ref{tab:major}. 
We make the following observations:

\textbf{Overall Performance.}
In general, compared to the other baselines, \ourmethod{} achieves the best performance. 
Specifically, in the multi-source transfer learning setting, \ourmethod{} outperforms the state-of-the-art (\text{OTQMS}) by an average of $1.3\%$ on \texttt{DomainNet} and $1.4\%$ on \texttt{Office-Home}.



\textbf{Sample-Based vs. Model-Based.}
In the multi-source transfer learning setting, methods based on sample consistently exhibit stronger performance than those relying on model-level aggregation.
Specifically, approaches that directly optimize a mixture of source and target samples—such as WADN, MADA, OTQMS, and \ourmethod{}—tend to outperform model-based techniques like H-ensemble and MCW, which combine predictions or parameters from multiple pretrained models.
A plausible explanation is that sample-based strategies allow the target model to interact directly with source data during training, enabling more flexible adaptation to target-specific distributions, whereas model-based methods are constrained by fixed source representations.

\textbf{Unified Optimization Outperforms Isolated Optimization.}
Among sample-based approaches, \ourmethod{} further improves upon existing baselines by jointly optimizing both the transfer quantities and the corresponding weights, while also deriving transfer weights that are explicitly dependent on the transfer quantities. In contrast, OTQMS focuses solely on optimizing transfer quantities, whereas WADN primarily adjusts weighting coefficients.


\textbf{Effectiveness of Limited Supervision.}
We further compare supervised and unsupervised multi-source transfer learning approaches under otherwise identical settings.
As shown in Table~\ref{tab:major}, even when unsupervised methods such as MSFDA and M3SDA leverage all available target samples—up to $1.3\times10^5$ instances in the Real domain of DomainNet—their performance remains inferior to supervised methods that utilize only a small number of labeled target samples (3450 in total).
This comparison suggests that a modest amount of target supervision can provide substantially more informative guidance than large quantities of unlabeled data, underscoring the practical value of few-shot supervision in multi-source transfer learning.

\textbf{Gains over Strong Baselines in Challenging Domains.}  
Beyond average accuracy improvements, \ourmethod{} consistently outperforms strong baselines in the most challenging domains.  For example, in the \texttt{DomainNet} experiments, \ourmethod{} improves over the previous best results in the Quickdraw target domain by 3.0\% absolute accuracy (36.8\% vs.\ 33.8\%) and in the Infograph target domain by 1.4\% (35.2\% vs.\ 33.8\%). 
These gains in visually
divergent domains demonstrate that \ourmethod{}'s sample-weighting strategy is particularly effective in mitigating domain shift. These results highlight the robustness of \ourmethod{} under severe domain shifts, where knowledge transfer is particularly challenging.


\subsection{Multi-Task Learning Performance}

The results for the multi-task learning setting are summarized in Table~\ref{tab:multitask_combined}. We make the following observations.

\textbf{Overall Performance Superiority.}  
Across both benchmarks, \ourmethod{} achieves the best overall average performance. 
On \texttt{DomainNet}, it obtains the highest mean accuracy of 51.9\%, improving upon the strongest baseline MoDo (51.2\%) by 0.7 percentage points.
Similarly, on \texttt{Office-Home}, \ourmethod{} achieves the best average accuracy of 79.1\%, surpassing MoDo (78.7\%) and MoCo (78.3\%) by 0.4 and 0.8 percentage points, respectively.

\textbf{Consistent and Balanced Multi-Task Performance.}
Across the ten target domains from the two benchmarks, \ourmethod{} achieves performance improvements on eight domains compared with the strongest competing methods, demonstrating broad effectiveness across heterogeneous tasks.
In particular, the gains are consistent on DomainNet, where \ourmethod{} outperforms all baselines across all six target domains (C, I, P, Q, R, and S), indicating that the improvements are not driven by a small subset of tasks.
On Office-Home, while the improvements are more selective, \ourmethod{} still attains the best overall average accuracy and achieves clear gains on two of the four target domains (Ar and Rw), while remaining competitive on the others.
Importantly, improvements on more challenging domains do not come at the expense of degraded performance on easier ones, suggesting effective mitigation of negative task interference.
Overall, these results highlight the ability of the proposed method to balance shared representation learning with task-specific optimization in multi-task settings.

\textbf{Extension of Weighting Theory to Multi-Task Learning.}
Although the proposed weighting theory is derived under a multi-source transfer learning formulation, our experiments demonstrate that it generalizes effectively to the multi-task learning setting. This observation validates the effectiveness and generality of the weights introduced by \ourmethod{}. Specifically, \ourmethod{} iteratively optimizes all model parameters and adaptively regulates cross-task influence through weights derived from our theoretical analysis and explicitly related to the transfer quantities.

\subsection{Robustness under Different Shot Settings}
We further investigate the robustness of \ourmethod{} with respect to the availability of target samples by gradually increasing the number of shots per class from 5 to 100.
Fig.~\ref{fig:shots} reports the results on the Office-Home dataset across all four target domains (Ar, Cl, Pr, and Rw), in comparison with Target-Only and AllSources $\cup$ Target.
Across all domains, \ourmethod{} maintains consistently competitive accuracy throughout the entire range of shot settings, exhibiting neither performance collapse in low-shot regimes nor degradation as more target samples are introduced.
In contrast to methods that rely heavily on either source aggregation or target-only supervision, \ourmethod{} exhibits stable and monotonic performance trends as the number of shots increases.
These observations indicate that the proposed approach adapts smoothly to different data regimes and remains robust under varying levels of target supervision.

\begin{figure}[!h]
    \centering
        \includegraphics[width=\linewidth]{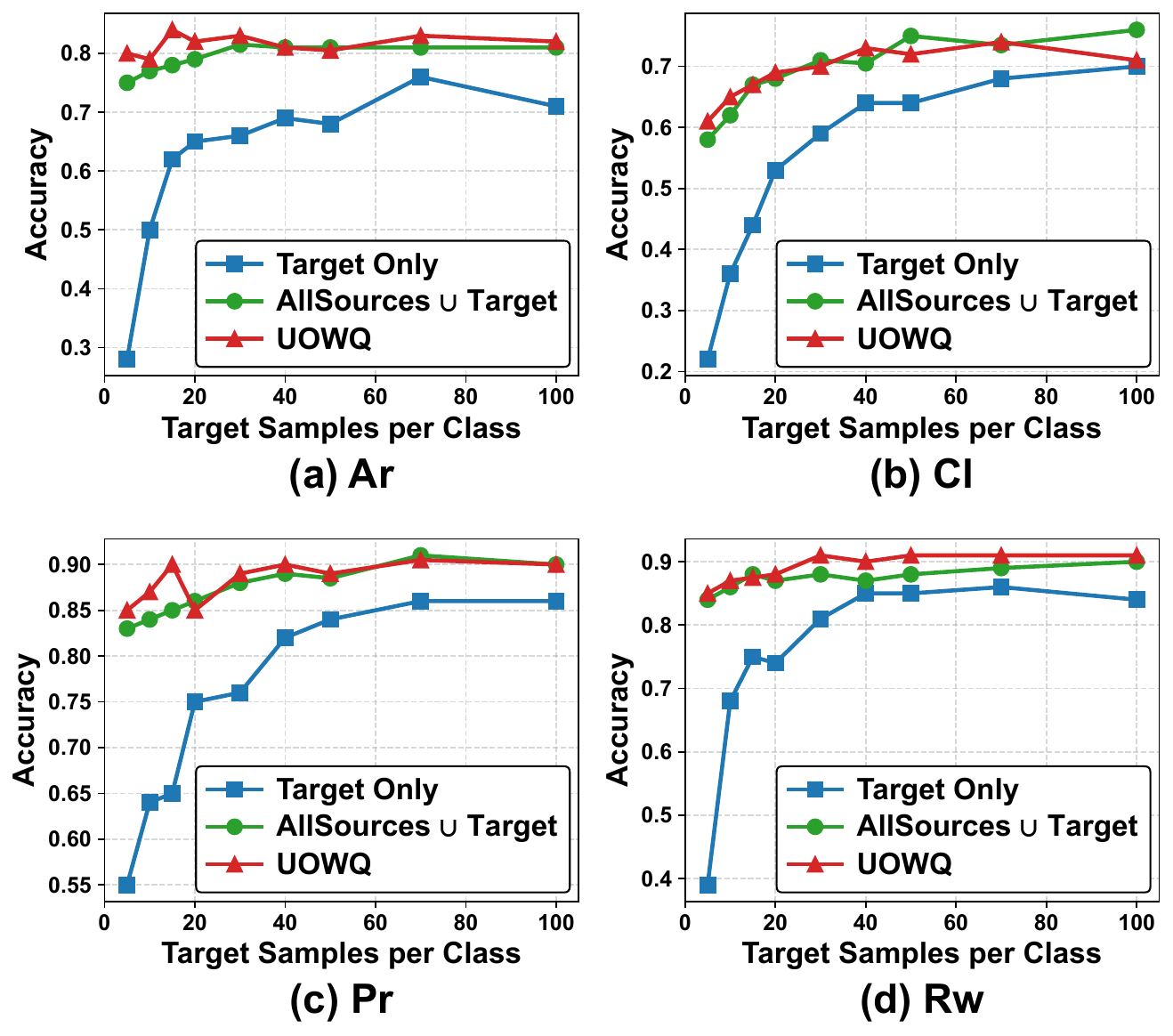}
    \caption{Performance comparison on the Office-Home dataset under varying target shot settings.}
    \label{fig:shots}
\end{figure}

\begin{figure*}[!t]
\centering

\begin{minipage}[t]{0.49\textwidth}
    \centering
    \begin{minipage}{0.48\linewidth}
        \centering
        \includegraphics[width=\linewidth]{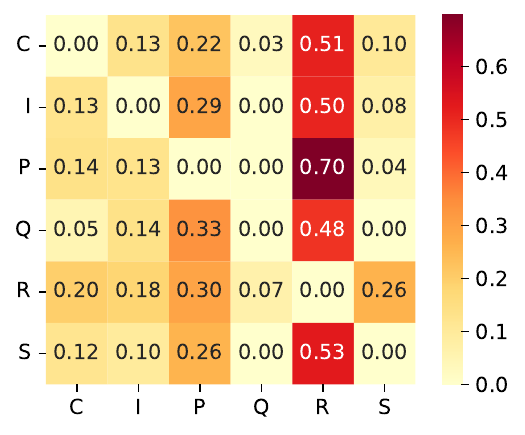}
        {\small (a) DomainNet: $\alpha$ weights}
    \end{minipage}\hfill
    \begin{minipage}{0.48\linewidth}
        \centering
        \includegraphics[width=\linewidth]{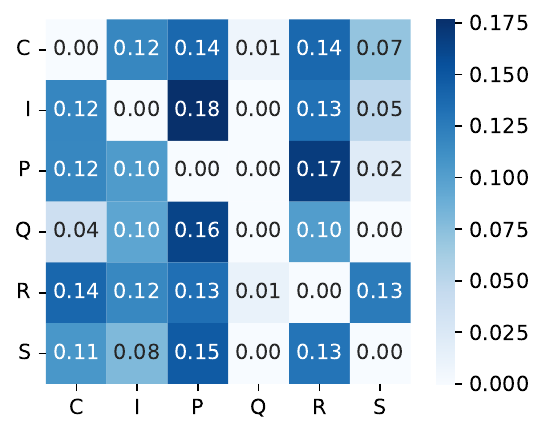}
        {\small (b) DomainNet: $w$ weights}
    \end{minipage}

    \vspace{3mm}

    \begin{minipage}{0.48\linewidth}
        \centering
        \includegraphics[width=\linewidth]{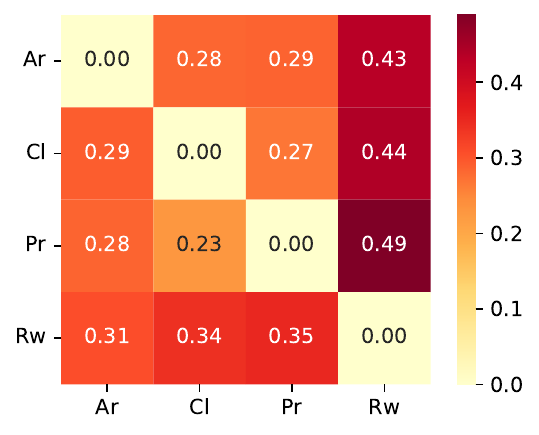}
        {\small (c) Office-Home: $\alpha$ weights}
    \end{minipage}\hfill
    \begin{minipage}{0.48\linewidth}
        \centering
        \includegraphics[width=\linewidth]{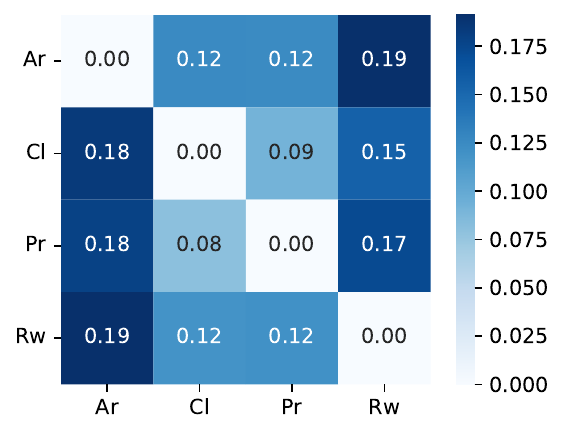}
        {\small (d) Office-Home: $w$ weights}
    \end{minipage}

    \captionof{figure}{Domain preference visualization on DomainNet and Office-Home.}
    \label{fig:heatmap_domain}
\end{minipage}
\hfill
\begin{minipage}[t]{0.49\textwidth}
    \centering
    \begin{minipage}{0.48\linewidth}
        \centering
        \includegraphics[width=\linewidth]{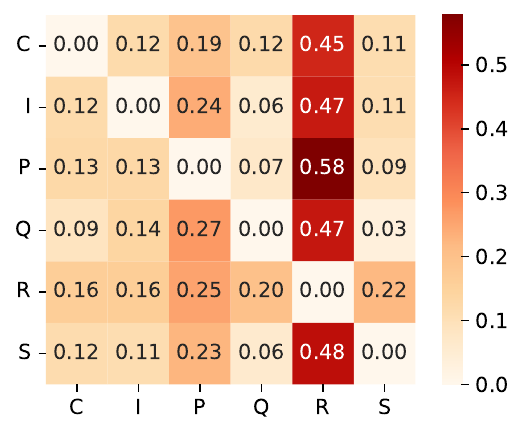}
        {\small (a) DomainNet: $\alpha$ weights}
    \end{minipage}\hfill
    \begin{minipage}{0.48\linewidth}
        \centering
        \includegraphics[width=\linewidth]{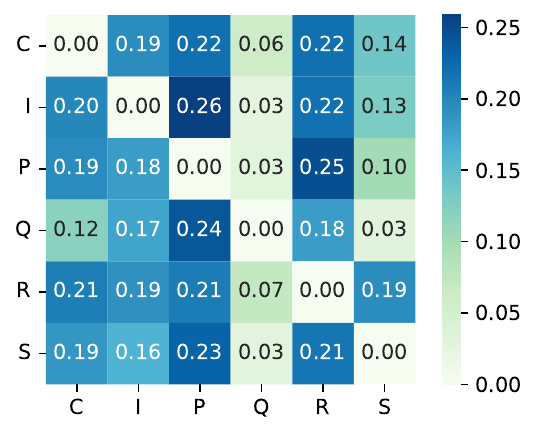}
        {\small (b) DomainNet: $w$ weights}
    \end{minipage}

    \vspace{3mm}

    \begin{minipage}{0.48\linewidth}
        \centering
        \includegraphics[width=\linewidth]{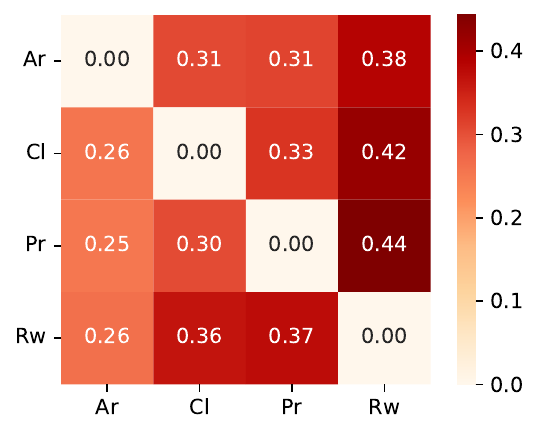}
        {\small (c) Office-Home: $\alpha$ weights}
    \end{minipage}\hfill
    \begin{minipage}{0.48\linewidth}
        \centering
        \includegraphics[width=\linewidth]{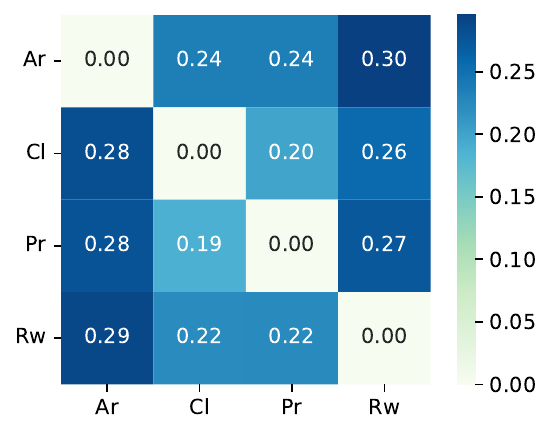}
        {\small (d) Office-Home: $w$ weights}
    \end{minipage}

    \captionof{figure}{Domain preference visualization on DomainNet and Office-Home trained with half of the source dataset.}
    \label{fig:heatmap_domain_2}
\end{minipage}
\end{figure*}

\subsection{Weight Visualization}
To analyze the domain preference learned by \ourmethod{}, we visualize the optimal coefficients $\alpha^*$ (as defined in~\eqref{eq_definition_optimal_alpha})
together with the base weights $w^*$(as defined in ~\eqref{eq_definition_optimal_weight}) in Fig.~\ref{fig:heatmap_domain},
where each row corresponds to a target domain and each column corresponds to a source domain.
On DomainNet, when Clipart and Painting are treated as target domains, the Real domain receives the largest $\alpha^*$ and $w^*$ coefficients. In addition, the Quickdraw domain receives relatively small weights for any target domain. These observed domain correlations are consistent with previous findings reported in~\cite{peng2019moment_M3SDA}. In Office-Home, the Real-World domain obtains the highest $\alpha^*$ for Art, Clipart, and Product targets, which is consistent with previous empirical observations in~\cite{liang2020we}. Overall, these results suggest that the quadratic optimization naturally assigns larger transfer coefficients to empirically more transferable source domains.   

Moreover, Fig.~\ref{fig:heatmap_domain_2} visualizes the learned domain preferences when only half of the source samples are used. Compared with Fig.~\ref{fig:heatmap_domain}, the relative patterns of the normalized coefficients $\alpha$ remain largely unchanged, indicating that the inferred domain similarity is stable. In contrast, the absolute source weights $w$ become noticeably larger across most source–target pairs. This behavior is consistent with our theory: when the source sample size $N_i$ is reduced, the source effect weakens, and the optimal solution compensates by assigning larger source weights to maintain a favorable bias–variance trade-off. 
From the \eqref{eq_definition_optimal_weight}, since $w^*_i=\frac{s^*{\alpha}^*_i}{N_i}$ explicitly depends on $N_i$, decreasing $N_i$ naturally leads to increased $w_i$.
Overall, the comparison between Fig.~\ref{fig:heatmap_domain} and Fig.~\ref{fig:heatmap_domain_2} shows that UOWQ adaptively adjusts source weights according to transfer quantities rather than relying on fixed weighting schemes.





\subsection{Computational Efficiency.}
We further compare the computational cost of different methods on the DomainNet benchmark, and report the wall-clock training time in Table~\ref{tab:time}.
Compared with MADA, which requires substantially longer training time across all target domains, both \allsource{} and \ourmethod{} exhibit significantly improved efficiency.
Notably, \ourmethod{} achieves comparable training time to \allsource{}, with only a marginal overhead introduced by the additional optimization procedure.
Across all domains, the average training time of \ourmethod{} (10:00:32) remains close to that of \allsource{} (09:50:10), and is substantially lower than MADA (17:07:55).
These results indicate that the proposed method introduces minimal computational overhead while providing consistent performance gains, demonstrating a favorable trade-off between effectiveness and efficiency.

\subsection{Static vs. Dynamic Weighting.}
Table~\ref{tab:dynamic} further examines the effect of dynamically updating source weights during training.
Compared with the static variant, which fixes the weights computed at initialization, the dynamic strategy consistently achieves higher accuracy across all Office-Home target domains.
The performance gap is particularly pronounced on more challenging domains such as Clipart and Product, indicating that continuously adapting the weights to the evolving target model is beneficial.
These results suggest that dynamic re-optimization enables the transfer strategy to better track changes in the target representation, thereby suppressing suboptimal or outdated source contributions and improving overall robustness under limited target supervision.

\subsection{Compatibility to LoRA-Based Training}

We further examine the compatibility of \ourmethod{} with LoRA. 
Specifically, we apply \ourmethod{} to multi-source transfer on the \textbf{Office-Home} dataset using the ViT-B backbone. 
Each domain is treated as a distinct source task, and LoRA modules with rank 8 are inserted into the projection layers of transformer blocks.  
As summarized in Table~\ref{tab:osw_lora}, \ourmethod{} consistently improves accuracy over baselines, confirming that the proposed theoretical principles remain effective under LoRA-based adaptation.


\section{Conclusion}

In this work, we propose \ourmethod{}, a theoretical framework for the unified optimization of source weights and transfer quantities in multi-source transfer learning.
By formulating the training process as a parameter estimation problem and conducting an asymptotic analysis of a K-L generalization error measure, we derive a principled method for computing the optimal source weights and transfer quantities. 
A notable theoretical insight is that, when the source weights are jointly optimized, the optimal strategy is to use all available source samples. This is because enlarging the transfer quantities reduces estimation variance, while adaptive weighting suppresses the bias caused by heterogeneous domains.
Building on these theoretical results, we further develop practical algorithms that assign source weights and update them dynamically during training for both multi-source transfer learning and multi-task learning scenarios. 
Extensive experiments on real-world datasets demonstrate that \ourmethod{} consistently improves performance over state-of-the-art baselines, particularly in few-shot and cross-domain settings. These results validate both the theoretical foundations and the practical advantages of our approach.

\bibliographystyle{IEEEtran}
\bibliography{ref}

@book{cover1999elements,
  title={Elements of information theory},
  author={Cover, Thomas M},
  year={1999},
  publisher={John Wiley \& Sons}
}

@inproceedings{wang2019characterizing,
  title={Characterizing and avoiding negative transfer},
  author={Wang, Zirui and Dai, Zihang and P{\'o}czos, Barnab{\'a}s and Carbonell, Jaime},
  booktitle={Proceedings of the IEEE/CVF conference on computer vision and pattern recognition},
  pages={11293--11302},
  year={2019}
}

@inproceedings{kendall2018multi,
  title={Multi-Task Learning Using Uncertainty to Weigh Losses for Scene Geometry and Semantics},
  author={Kendall, Alex and Gal, Yarin and Cipolla, Roberto},
  booktitle={Proceedings of the IEEE Conference on Computer Vision and Pattern Recognition (CVPR)},
  year={2018},
  url={https://arxiv.org/abs/1705.07115}
}

@inproceedings{chen2018gradnorm,
  title={GradNorm: Gradient Normalization for Adaptive Loss Balancing in Deep Multitask Networks},
  author={Chen, Zhao and Badrinarayanan, Vijay and Lee, Chen-Yu and Rabinovich, Andrew},
  booktitle={International Conference on Machine Learning (ICML)},
  year={2018},
  url={https://arxiv.org/abs/1711.02257}
}

@inproceedings{liu2019end,
  title={End-to-End Multi-Task Learning with Attention},
  author={Liu, Shikun and Johns, Edward and Davison, Andrew J},
  booktitle={Proceedings of the IEEE Conference on Computer Vision and Pattern Recognition (CVPR)},
  year={2019},
  url={https://arxiv.org/abs/1803.10704}
}

@article{liu2021task,
  title={Task Weighting in Multi-Task Learning via Gradient Balancing},
  author={Liu, Bo and Wei, Yingyu and Zhang, Yu and Yang, Qiang},
  journal={NeurIPS},
  year={2021}
}

@article{wang2020multisource,
  title={Multi-Source Domain Adaptation with Mixture of Experts},
  author={Wang, Ximei and Li, Jinghan and Jin, Lei and Pang, Jiangmiao and Luo, Dongsheng},
  journal={ECCV},
  year={2020}
}

@article{zhou2021deep,
  title={Deep Multi-Source Transfer Learning for Cross-Domain Recommendations},
  author={Zhou, Wenxuan and Wen, Hongwei and Zhang, Yongfeng},
  journal={IEEE TKDE},
  year={2021}
}

@inproceedings{sener2018multi,
  title={Multi-Task Learning as Multi-Objective Optimization},
  author={Sener, Ozan and Koltun, Vladlen},
  booktitle={Advances in Neural Information Processing Systems (NeurIPS)},
  year={2018},
  url={https://arxiv.org/abs/1810.04650}
}

@inproceedings{yu2020gradient,
  title={Gradient Surgery for Multi-Task Learning},
  author={Yu, Tianhe and Kumar, Saurabh and Gupta, Abhishek and Levine, Sergey and Hausman, Karol and Finn, Chelsea},
  booktitle={International Conference on Learning Representations (ICLR)},
  year={2020},
  url={https://arxiv.org/abs/2001.06782}
}

@inproceedings{liu2021conflict,
  title={Conflict-Averse Gradient Descent for Multi-Task Learning},
  author={Liu, Bo and Liu, Xingchao and Jin, Xiaoxi and Stone, Peter and Liu, Qiang},
  booktitle={Advances in Neural Information Processing Systems (NeurIPS)},
  year={2021},
  url={https://arxiv.org/abs/2010.14030}
}

@inproceedings{sun2020mtadam,
  title={Adaptive Task Sampling for Meta-Learning},
  author={Sun, Qianru and Liu, Yaoyao and Chua, Tat-Seng and Schiele, Bernt},
  booktitle={Advances in Neural Information Processing Systems (NeurIPS)},
  year={2020},
  url={https://arxiv.org/abs/2006.15586}
}

@inproceedings{maninis2022rotograd,
  title={RotoGrad: Gradient Homogenization in Multitask Learning},
  author={Maninis, Kevis-Kokitsi and Radosavovic, Ilija and Kokkinos, Iasonas},
  booktitle={International Conference on Learning Representations (ICLR)},
  year={2022},
  url={https://arxiv.org/abs/2103.14010},
  note={Notable Top 5\% Paper}
}

@article{you2021logme,
  title={LogME: Practical Assessment of Pre-trained Models for Transfer Learning},
  author={You, Kaichao and Liu, Yong and Wang, Jianmin and Long, Mingsheng},
  journal={International Conference on Machine Learning (ICML)},
  year={2021}
}

@article{tan2020survey,
  title={A Survey on Deep Transfer Learning},
  author={Tan, Chuanqi and Sun, Fuchun and Kong, Tao and Zhang, Wenchang and Yang, Chao and Liu, Chunfang},
  journal={Neural Networks},
  volume={121},
  pages={135--151},
  year={2020}
}

@article{nguyen2020leep,
  title={LEEP: A New Measure to Evaluate Transferability of Learned Representations},
  author={Nguyen, Cuong and Hassner, Tal and Seeger, Matthias and Archambeau, Cedric},
  journal={International Conference on Machine Learning (ICML)},
  year={2020}
}

@article{liu2021geometric,
  title={Geometric-Based Domain Adaptation for Transfer Learning},
  author={Liu, Hong and Long, Mingsheng and Wang, Jianmin and Wang, Yu},
  journal={AAAI Conference on Artificial Intelligence},
  year={2021}
}

@article{chen2022otce,
  title={OTCE: Optimal Transport for Conditional Entropy in Transfer Learning},
  author={Chen, Xinyang and Wang, Sinong and Wang, Jianfeng and Huang, Yelong},
  journal={NeurIPS},
  year={2022}
}

@inproceedings{zheng2021libra,
  title={Libra: Balancing Tasks for Multi-Task Learning},
  author={Zheng, Xiong and Liu, Yang and Hua, Yutong and Tian, Yuchao and Zhang, Tao},
  booktitle={International Conference on Machine Learning (ICML)},
  year={2021},
  url={https://arxiv.org/abs/2107.07018}
}

@inproceedings{du2023game,
  title={Game-Theoretic Optimization for Multi-Task Learning},
  author={Du, Yingjun and Lin, Junyi and Zhou, Pan},
  booktitle={Advances in Neural Information Processing Systems (NeurIPS)},
  year={2023},
  url={https://arxiv.org/abs/2302.02842}
}

@inproceedings{chen2023flix,
  title={FLIX: A Simple and Communication-Efficient Alternative to Local Methods in Federated Learning},
  author={Chen, Tianyi and Sun, Yuejiao and Yin, Wotao},
  booktitle={International Conference on Machine Learning (ICML)},
  year={2023},
  url={https://arxiv.org/abs/2302.02842}
}

@inproceedings{triantafillou2020meta,
  title={Meta-Dataset: A Dataset of Datasets for Learning to Learn from Few Examples},
  author={Triantafillou, Eleni and Zhu, Tyler and Dumoulin, Vincent and Lamblin, Pascal and Evci, Utku and Xu, Kelvin and Goroshin, Ross and Gelada, Carles and Swersky, Kevin and Manzagol, Pierre-Antoine and others},
  booktitle={International Conference on Learning Representations (ICLR)},
  year={2020},
  url={https://arxiv.org/abs/1903.03096}
}

@article{caruana1997multitask,
  title={Multitask learning},
  author={Caruana, Rich},
  journal={Machine learning},
  volume={28},
  number={1},
  pages={41--75},
  year={1997},
  publisher={Springer}
}

@inproceedings{wu2024h_Hensemble,
  title={H-ensemble: An Information Theoretic Approach to Reliable Few-Shot Multi-Source-Free Transfer},
  author={Wu, Yanru and Wang, Jianning and Wang, Weida and Li, Yang},
  booktitle={Proceedings of the AAAI Conference on Artificial Intelligence},
  volume={38},
  number={14},
  pages={15970--15978},
  year={2024}
}

@inproceedings{zhang2024revisiting_MADA,
  title={Revisiting the domain shift and sample uncertainty in multi-source active domain transfer},
  author={Zhang, Wenqiao and Lv, Zheqi and Zhou, Hao and Liu, Jia-Wei and Li, Juncheng and Li, Mengze and Li, Yunfei and Zhang, Dongping and Zhuang, Yueting and Tang, Siliang},
  booktitle={Proceedings of the IEEE/CVF Conference on Computer Vision and Pattern Recognition},
  pages={16751--16761},
  year={2024}
}

@article{lee2019learning_MCW,
  title={Learning new tricks from old dogs: Multi-source transfer learning from pre-trained networks},
  author={Lee, Joshua and Sattigeri, Prasanna and Wornell, Gregory},
  journal={Advances in neural information processing systems},
  volume={32},
  year={2019}
}

@inproceedings{shui2021aggregating_WADN,
  title={Aggregating from multiple target-shifted sources},
  author={Shui, Changjian and Li, Zijian and Li, Jiaqi and Gagn{\'e}, Christian and Ling, Charles X and Wang, Boyu},
  booktitle={International Conference on Machine Learning},
  pages={9638--9648},
  year={2021},
  organization={PMLR}
}

@inproceedings{fernando2022mitigating,
  title={Mitigating Gradient Bias in Multi-objective Learning: A Provably Convergent Stochastic Approach},
  author={Fernando, Heshan and Shen, Han and Liu, Miao and Chaudhury, Subhajit and Murugesan, Keerthiram and Chen, Tianyi},
  booktitle={ICLR},
  year={2023},
}

@article{chen2023three,
  title={Three-way trade-off in multi-objective learning: Optimization, generalization and conflict-avoidance},
  author={Chen, Lisha and Fernando, Heshan and Ying, Yiming and Chen, Tianyi},
  journal={Advances in Neural Information Processing Systems},
  volume={36},
  pages={70045--70093},
  year={2023}
}

@inproceedings{deng2009imagenet,
  title={Imagenet: A large-scale hierarchical image database},
  author={Deng, Jia and Dong, Wei and Socher, Richard and Li, Li-Jia and Li, Kai and Fei-Fei, Li},
  booktitle={2009 IEEE conference on computer vision and pattern recognition},
  pages={248--255},
  year={2009},
  organization={Ieee}
}

@article{krizhevsky2012imagenet,
  title={Imagenet classification with deep convolutional neural networks},
  author={Krizhevsky, Alex and Sutskever, Ilya and Hinton, Geoffrey E},
  journal={Advances in neural information processing systems},
  volume={25},
  year={2012}
}

@inproceedings{he2016deep_resnet,
  title={Deep residual learning for image recognition},
  author={He, Kaiming and Zhang, Xiangyu and Ren, Shaoqing and Sun, Jian},
  booktitle={Proceedings of the IEEE conference on computer vision and pattern recognition},
  pages={770--778},
  year={2016}
}

@misc{rw2019timm_vits,
  author = {Ross Wightman},
  title = {PyTorch Image Models},
  year = {2019},
  publisher = {GitHub},
  journal = {GitHub repository},
  doi = {10.5281/zenodo.4414861},
  howpublished = {\url{https://github.com/huggingface/pytorch-image-models}}
}

@inproceedings{venkateswara2017deep,
  title={Deep hashing network for unsupervised domain adaptation},
  author={Venkateswara, Hemanth and Eusebio, Jose and Chakraborty, Shayok and Panchanathan, Sethuraman},
  booktitle={Proceedings of the IEEE conference on computer vision and pattern recognition},
  pages={5018--5027},
  year={2017}
}

@article{kingma2014adam,
  title={Adam: A method for stochastic optimization},
  author={Kingma, Diederik P and Ba, Jimmy},
  journal={arXiv preprint arXiv:1412.6980},
  year={2014}
}

@inproceedings{liang2020we,
  title={Do we really need to access the source data? source hypothesis transfer for unsupervised domain adaptation},
  author={Liang, Jian and Hu, Dapeng and Feng, Jiashi},
  booktitle={International conference on machine learning},
  pages={6028--6039},
  year={2020},
  organization={PMLR}
}

@article{mansour2008domain,
  title={Domain adaptation with multiple sources},
  author={Mansour, Yishay and Mohri, Mehryar and Rostamizadeh, Afshin},
  journal={Advances in neural information processing systems},
  volume={21},
  year={2008}
}

@inproceedings{long2015learning,
  title={Learning transferable features with deep adaptation networks},
  author={Long, Mingsheng and Cao, Yue and Wang, Jianmin and Jordan, Michael},
  booktitle={International conference on machine learning},
  pages={97--105},
  year={2015},
  organization={PMLR}
}

@article{zhao2018adversarial,
  title={Adversarial multiple source domain adaptation},
  author={Zhao, Han and Zhang, Shanghang and Wu, Guanhang and Moura, Jos{\'e} MF and Costeira, Joao P and Gordon, Geoffrey J},
  journal={Advances in neural information processing systems},
  volume={31},
  year={2018}
}

@inproceedings{chen2018domain,
  title={Domain adaptive faster r-cnn for object detection in the wild},
  author={Chen, Yuhua and Li, Wen and Sakaridis, Christos and Dai, Dengxin and Van Gool, Luc},
  booktitle={Proceedings of the IEEE conference on computer vision and pattern recognition},
  pages={3339--3348},
  year={2018}
}

@article{tong2021mathematical,
  title={A mathematical framework for quantifying transferability in multi-source transfer learning},
  author={Tong, Xinyi and Xu, Xiangxiang and Huang, Shao-Lun and Zheng, Lizhong},
  journal={Advances in Neural Information Processing Systems},
  volume={34},
  pages={26103--26116},
  year={2021}
}

@article{bu2020tightening,
  title={Tightening mutual information-based bounds on generalization error},
  author={Bu, Yuheng and Zou, Shaofeng and Veeravalli, Venugopal V},
  journal={IEEE Journal on Selected Areas in Information Theory},
  volume={1},
  number={1},
  pages={121--130},
  year={2020},
  publisher={IEEE}
}

@article{serte2022deep,
  title={Deep learning in medical imaging: A brief review},
  author={Serte, Sertan and Serener, Ali and Al-Turjman, Fadi},
  journal={Transactions on Emerging Telecommunications Technologies},
  volume={33},
  number={10},
  pages={e4080},
  year={2022},
  publisher={Wiley Online Library}
}

@inproceedings{devlin2019bert,
  title={Bert: Pre-training of deep bidirectional transformers for language understanding},
  author={Devlin, Jacob and Chang, Ming-Wei and Lee, Kenton and Toutanova, Kristina},
  booktitle={Proceedings of the 2019 conference of the North American chapter of the association for computational linguistics: human language technologies, volume 1 (long and short papers)},
  pages={4171--4186},
  year={2019}
}

@article{martens2020new,
  title={New insights and perspectives on the natural gradient method},
  author={Martens, James},
  journal={Journal of Machine Learning Research},
  volume={21},
  number={146},
  pages={1--76},
  year={2020}
}

@article{osawa2023pipefisher,
  title={PipeFisher: Efficient training of large language models using pipelining and Fisher information matrices},
  author={Osawa, Kazuki and Li, Shigang and Hoefler, Torsten},
  journal={Proceedings of Machine Learning and Systems},
  volume={5},
  pages={708--727},
  year={2023}
}

@article{zhao2021madan,
  title={MADAN: Multi-source adversarial domain aggregation network for domain adaptation},
  author={Zhao, Sicheng and Li, Bo and Xu, Pengfei and Yue, Xiangyu and Ding, Guiguang and Keutzer, Kurt},
  journal={International Journal of Computer Vision},
  volume={129},
  number={8},
  pages={2399--2424},
  year={2021},
  publisher={Springer}
}

@inproceedings{guo2020multi,
  title={Multi-source domain adaptation for text classification via distancenet-bandits},
  author={Guo, Han and Pasunuru, Ramakanth and Bansal, Mohit},
  booktitle={Proceedings of the AAAI conference on artificial intelligence},
  volume={34},
  number={05},
  pages={7830--7838},
  year={2020}
}

@inproceedings{li2021dynamic,
  title={Dynamic transfer for multi-source domain adaptation},
  author={Li, Yunsheng and Yuan, Lu and Chen, Yinpeng and Wang, Pei and Vasconcelos, Nuno},
  booktitle={Proceedings of the IEEE/CVF Conference on Computer Vision and Pattern Recognition},
  pages={10998--11007},
  year={2021}
}

@article{wan2022uav,
  title={UAV swarm based radar signal sorting via multi-source data fusion: A deep transfer learning framework},
  author={Wan, Liangtian and Liu, Rong and Sun, Lu and Nie, Hansong and Wang, Xianpeng},
  journal={Information Fusion},
  volume={78},
  pages={90--101},
  year={2022},
  publisher={Elsevier}
}

@article{zhao2024more,
  title={More is Better: Deep Domain Adaptation with Multiple Sources},
  author={Zhao, Sicheng and Chen, Hui and Huang, Hu and Xu, Pengfei and Ding, Guiguang},
  journal={arXiv preprint arXiv:2405.00749},
  year={2024}
}

@article{zhuang2020comprehensive,
  title={A comprehensive survey on transfer learning},
  author={Zhuang, Fuzhen and Qi, Zhiyuan and Duan, Keyu and Xi, Dongbo and Zhu, Yongchun and Zhu, Hengshu and Xiong, Hui and He, Qing},
  journal={Proceedings of the IEEE},
  volume={109},
  number={1},
  pages={43--76},
  year={2020},
  publisher={Ieee}
}

@article{li2021multi,
  title={Multi-source contribution learning for domain adaptation},
  author={Li, Keqiuyin and Lu, Jie and Zuo, Hua and Zhang, Guangquan},
  journal={IEEE Transactions on Neural Networks and Learning Systems},
  volume={33},
  number={10},
  pages={5293--5307},
  year={2021},
  publisher={IEEE}
}

@inproceedings{long2014transfer,
  title={Transfer joint matching for unsupervised domain adaptation},
  author={Long, Mingsheng and Wang, Jianmin and Ding, Guiguang and Sun, Jiaguang and Yu, Philip S},
  booktitle={Proceedings of the IEEE conference on computer vision and pattern recognition},
  pages={1410--1417},
  year={2014}
}

@article{courty2016optimal,
  title={Optimal transport for domain adaptation},
  author={Courty, Nicolas and Flamary, R{\'e}mi and Tuia, Devis and Rakotomamonjy, Alain},
  journal={IEEE transactions on pattern analysis and machine intelligence},
  volume={39},
  number={9},
  pages={1853--1865},
  year={2016},
  publisher={IEEE}
}

@article{zhang2022survey,
  title={A survey on negative transfer},
  author={Zhang, Wen and Deng, Lingfei and Zhang, Lei and Wu, Dongrui},
  journal={IEEE/CAA Journal of Automatica Sinica},
  volume={10},
  number={2},
  pages={305--329},
  year={2022},
  publisher={IEEE}
}

@inproceedings{jain2023data,
  title={A data-based perspective on transfer learning},
  author={Jain, Saachi and Salman, Hadi and Khaddaj, Alaa and Wong, Eric and Park, Sung Min and M{\k{a}}dry, Aleksander},
  booktitle={Proceedings of the IEEE/CVF Conference on Computer Vision and Pattern Recognition},
  pages={3613--3622},
  year={2023}
}

@inproceedings{
zhang2025theoretical,
title={A High-Dimensional Statistical Method for Optimizing Transfer Quantities in Multi-Source Transfer Learning},
author={Qingyue Zhang and Haohao Fu and Guanbo Huang and Yaoyuan Liang and Chang Chu and Tianren Peng and Yanru Wu and Qi Li and Yang Li and Shao-Lun Huang},
booktitle={The Thirty-ninth Annual Conference on Neural Information Processing Systems},
year={2025},
url={https://openreview.net/forum?id=fsTj0BNxyH}
}

@article{li2024scalable,
  title={Scalable Fine-tuning from Multiple Data Sources: A First-Order Approximation Approach},
  author={Li, Dongyue and Zhang, Ziniu and Wang, Lu and Zhang, Hongyang R},
  journal={Findings of the Association for Computational Linguistics: EMNLP 2024},
  year={2024},
  publisher={Association for Computational Linguistics}
}

@inproceedings{han2023discriminability_DATE,
  title={Discriminability and transferability estimation: a bayesian source importance estimation approach for multi-source-free domain adaptation},
  author={Han, Zhongyi and Zhang, Zhiyan and Wang, Fan and He, Rundong and Su, Wan and Xi, Xiaoming and Yin, Yilong},
  booktitle={Proceedings of the AAAI conference on artificial intelligence},
  volume={37},
  number={6},
  pages={7811--7820},
  year={2023}
}

@article{lin2022reasonable_RGW,
  title={Reasonable Effectiveness of Random Weighting: A Litmus Test for Multi-Task Learning},
  author={Lin, Baijiong and Ye, Feiyang and Zhang, Yu and Tsang, Ivor},
  journal={Transactions on Machine Learning Research},
  year={2022}
}

@article{raghu2019transfusion,
  title={Transfusion: Understanding transfer learning for medical imaging},
  author={Raghu, Maithra and Zhang, Chiyuan and Kleinberg, Jon and Bengio, Samy},
  journal={Advances in neural information processing systems},
  volume={32},
  year={2019}
}

@inproceedings{zhang2025apwdsit,
    author    = {Qingyue Zhang and Chang Chu and Haohao Fu and Tianren Peng and Shao{-}Lun Huang},
    title     = {Asymptotic Analysis for Optimal Source Weights in Multi-Source Transfer Learning},
    booktitle = {Proceedings of APWDSIT 2025},
    year      = {2025}
}

@inproceedings{shen2023balancingMSFDA,
  title={On balancing bias and variance in unsupervised multi-source-free domain adaptation},
  author={Shen, Maohao and Bu, Yuheng and Wornell, Gregory W},
  booktitle={International Conference on Machine Learning},
  pages={30976--30991},
  year={2023},
  organization={PMLR}
}

@inproceedings{peng2019moment_M3SDA,
  title={Moment matching for multi-source domain adaptation},
  author={Peng, Xingchao and Bai, Qinxun and Xia, Xide and Huang, Zijun and Saenko, Kate and Wang, Bo},
  booktitle={Proceedings of the IEEE/CVF international conference on computer vision},
  pages={1406--1415},
  year={2019}
}

@book{van2000asymptotic,
  title={Asymptotic statistics},
  author={Van der Vaart, Aad W},
  volume={3},
  year={2000},
  publisher={Cambridge university press}
}

@ARTICLE{tpami2023_zhu_Reinforcelearning_transfer,
  author={Zhu, Zhuangdi and Lin, Kaixiang and Jain, Anil K. and Zhou, Jiayu},
  journal={IEEE Transactions on Pattern Analysis and Machine Intelligence}, 
  title={Transfer Learning in Deep Reinforcement Learning: A Survey}, 
  year={2023},
  volume={45},
  number={11},
  pages={13344-13362},
  doi={10.1109/TPAMI.2023.3292075}}

@article{harremoes2011pairs,
  title={On pairs of $ f $-divergences and their joint range},
  author={Harremo{\"e}s, Peter and Vajda, Igor},
  journal={IEEE Transactions on Information Theory},
  volume={57},
  number={6},
  pages={3230--3235},
  year={2011},
  publisher={IEEE}
}

@article{xiao2024selective,
  title={Selective random walk for transfer learning in heterogeneous label spaces},
  author={Xiao, Qiao and Zhang, Yu and Yang, Qiang},
  journal={IEEE Transactions on Pattern Analysis and Machine Intelligence},
  volume={46},
  number={6},
  pages={4476--4488},
  year={2024},
  publisher={IEEE}
}

\begin{IEEEbiography}[{\includegraphics[width=1in,height=1.25in]{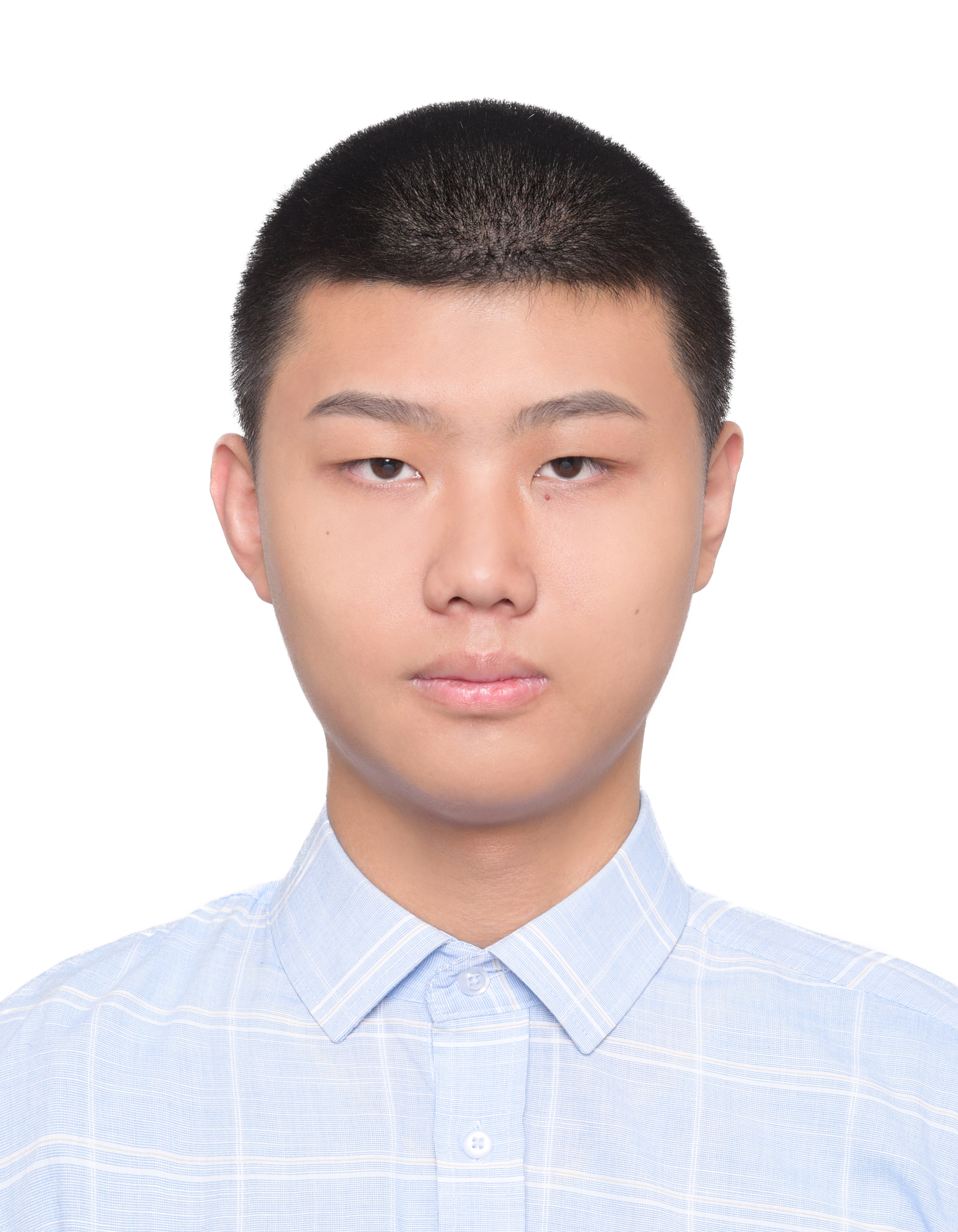}}]
{Qingyue Zhang} (Student Member, IEEE)
is currently pursuing the Ph.D. degree with the Tsinghua Shenzhen International Graduate School, Tsinghua University, Shenzhen, China. His current research interests include transfer learning and LLM.
\end{IEEEbiography}
\begin{IEEEbiography}[{\includegraphics[width=1in,height=1.25in]{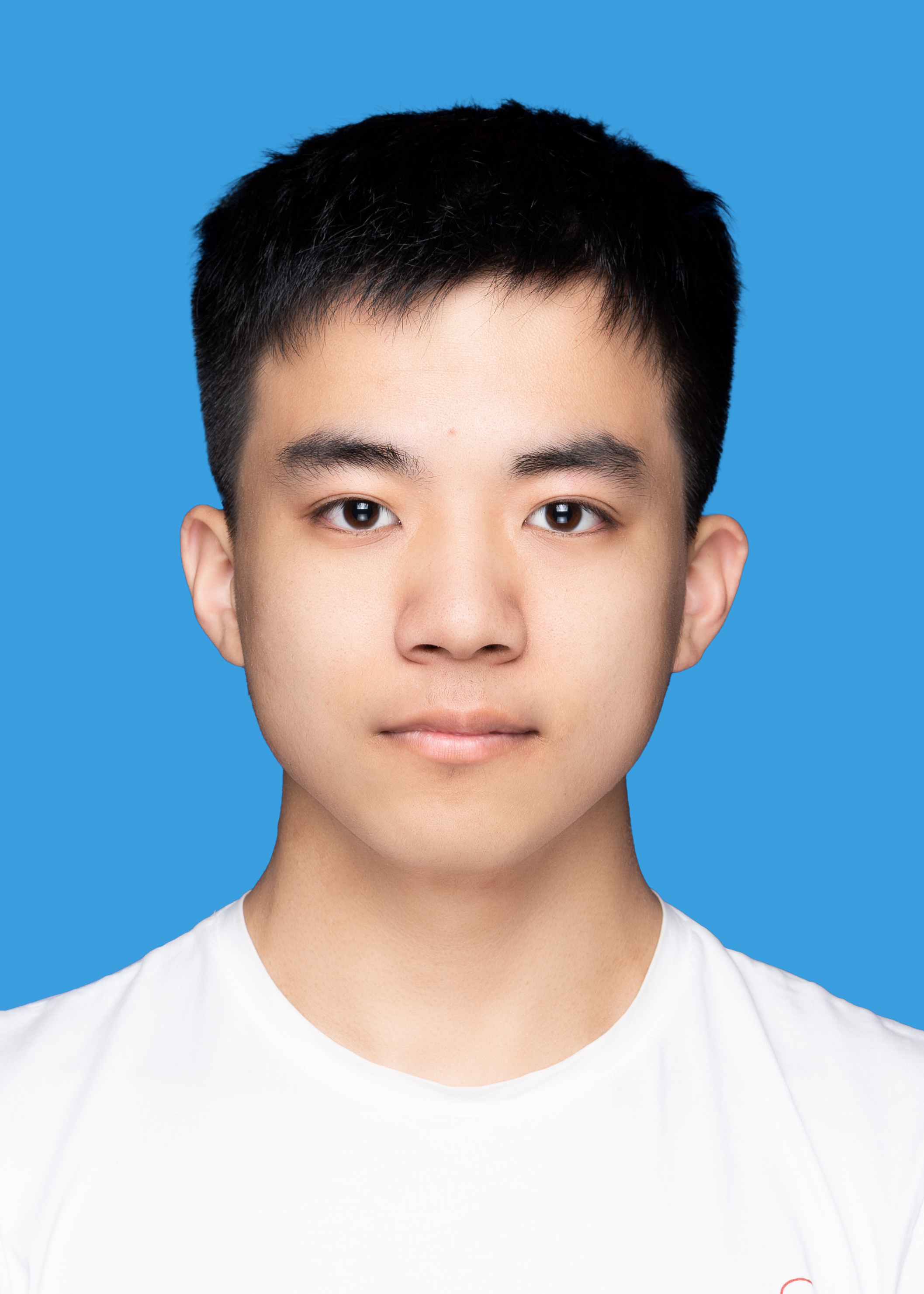}}]
{Chang Chu} (Student Member, IEEE)
is currently pursuing the M.S. degree with the Tsinghua Shenzhen International Graduate School, Tsinghua University, Shenzhen, China. His current research interests include transfer learning and class imbalance.  
\end{IEEEbiography}
\begin{IEEEbiography}[{\includegraphics[width=1in,height=1.25in]{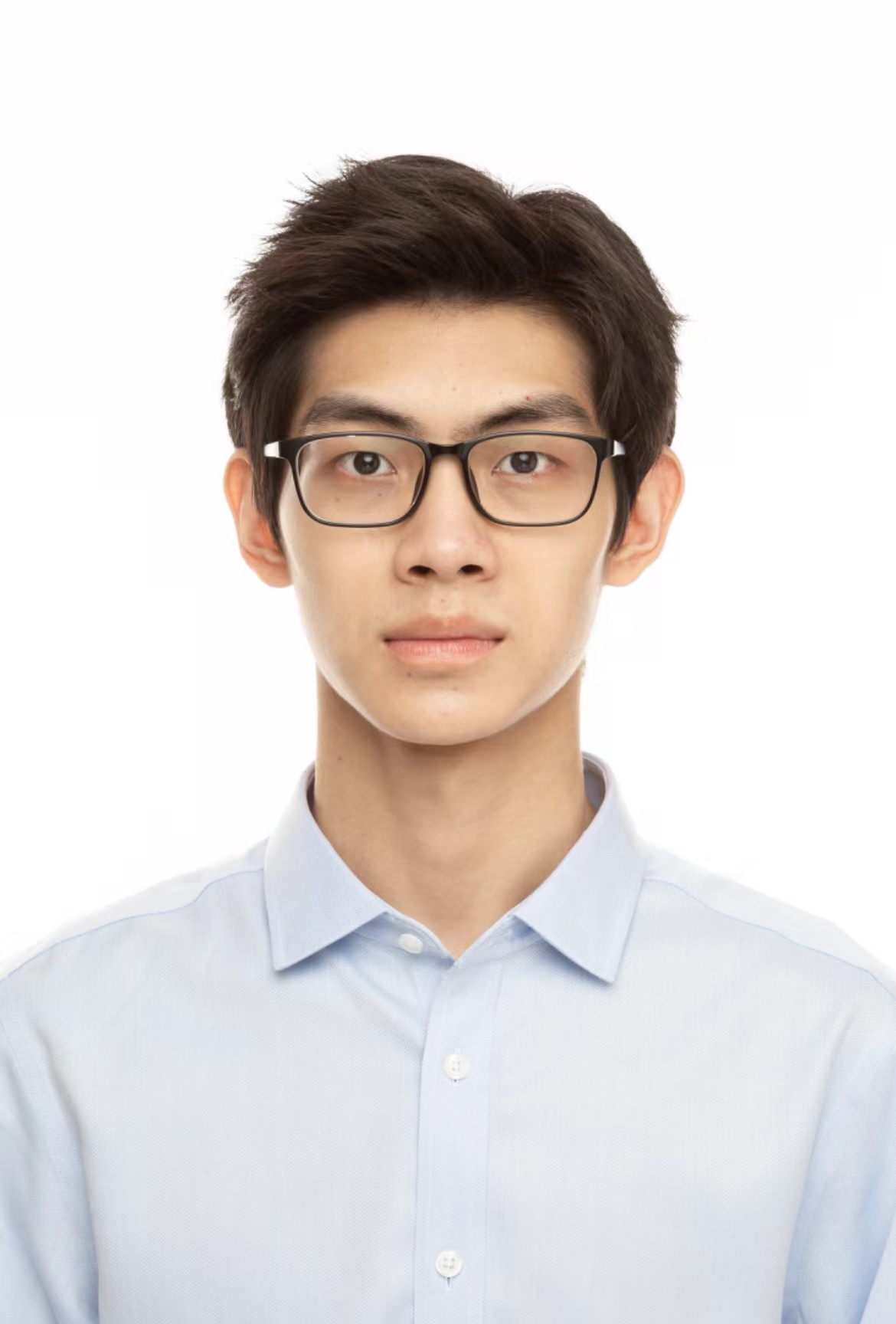}}]
{Haohao Fu} (Student Member, IEEE)
 is currently pursuing the Ph.D. degree with the Tsinghua Shenzhen International Graduate School, Tsinghua University, Shenzhen, China. His current research interests include information theory and machine learning.
\end{IEEEbiography}
\begin{IEEEbiography}[{\includegraphics[width=1in,height=1.25in]{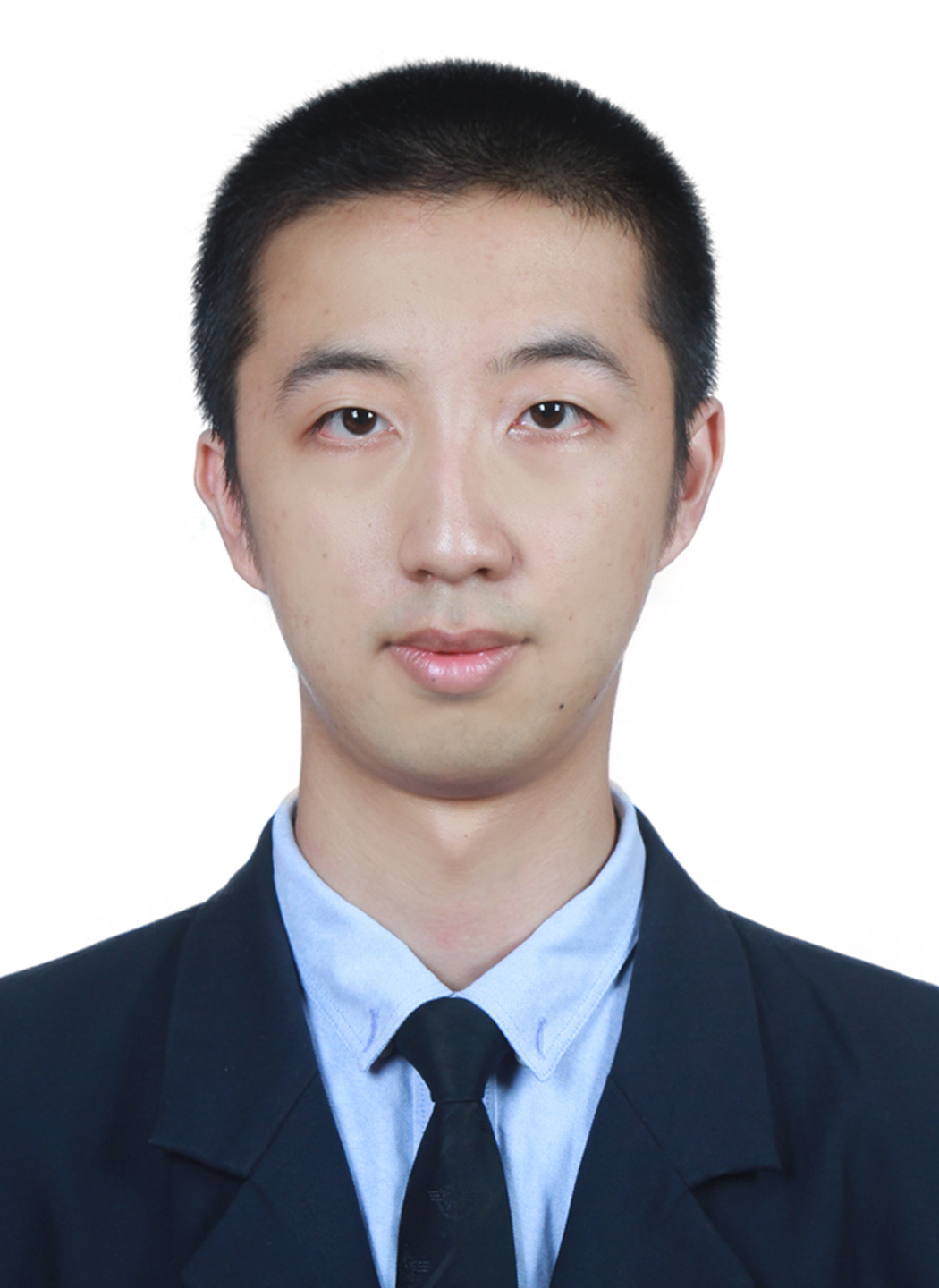}}]
{Tianren Peng} (Student Member, IEEE)
is currently pursuing the Ph.D. degree with the Tsinghua Shenzhen International Graduate School, Tsinghua University, Shenzhen, China.
His research interests include machine learning theory and its application in multi-modal learning and transfer learning.
\end{IEEEbiography}
\begin{IEEEbiography}[{\includegraphics[width=1in,height=1.25in]{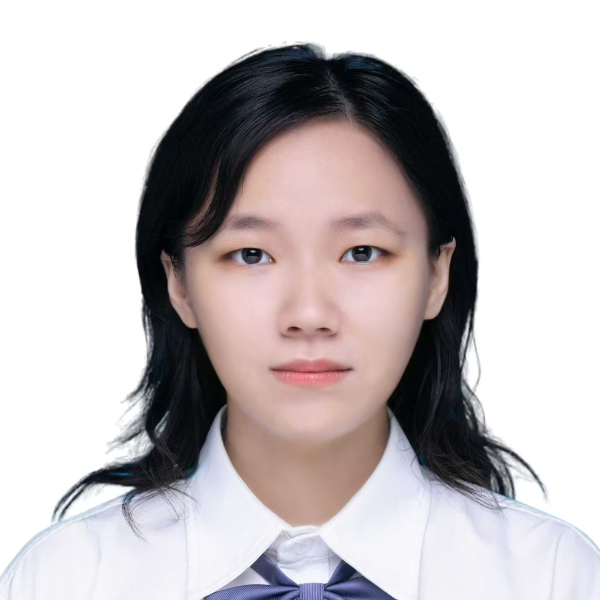}}]
{Yanru Wu} is currently pursuing the Ph.D. degree with the Tsinghua Shenzhen International Graduate School, Tsinghua University, Shenzhen, China. Her current research interests include transfer learning and continual learning.
\end{IEEEbiography}
\begin{IEEEbiography}[{\includegraphics[width=1in,height=1.25in]{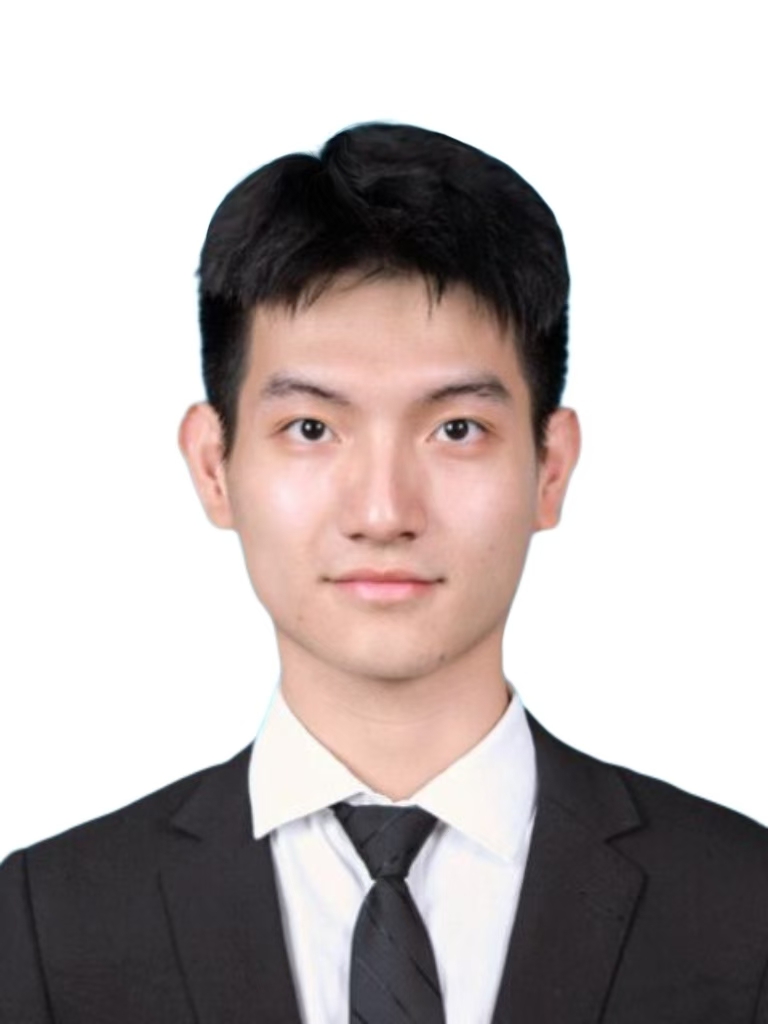}}]
{Guanbo Huang}
is currently pursuing the M.S. degree with the Tsinghua Shenzhen International Graduate School, Tsinghua University, Shenzhen, China.
His research interests include transfer learning, AIGC, RL, MLLM and Agent application.
\end{IEEEbiography}

\begin{IEEEbiography}[{\includegraphics[width=1in,height=1.25in]{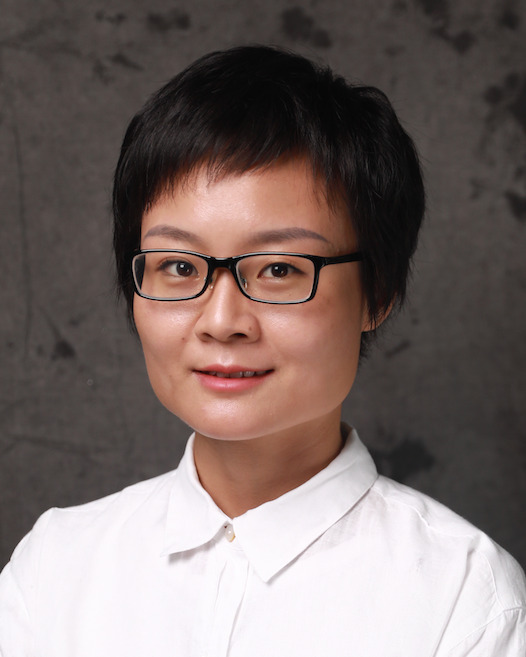}}]
{Yang Li} (Member, IEEE) is an Associate Professor with the School of Artificial Intelligence, The Chinese University of Hong Kong, Shenzhen. Prior to joining CUHK-Shenzhen, she was a postdoctoral researcher at the Tsinghua–Berkeley Shenzhen Institute and later worked with the Institute of Data and Information, Tsinghua Shenzhen International Graduate School. She received the B.A. degree in computer science and mathematics from the Smith College, Northampton, MA, USA, in 2011, and the Ph.D. degree in computer science from Stanford University, Stanford, CA, USA, in 2017. Her research interests include transfer learning, domain generalization, and interpretable representation learning, with applications to medical image analysis and spatio-temporal data modeling.
\end{IEEEbiography}

\begin{IEEEbiography}[{\includegraphics[width=1in,height=1.25in]{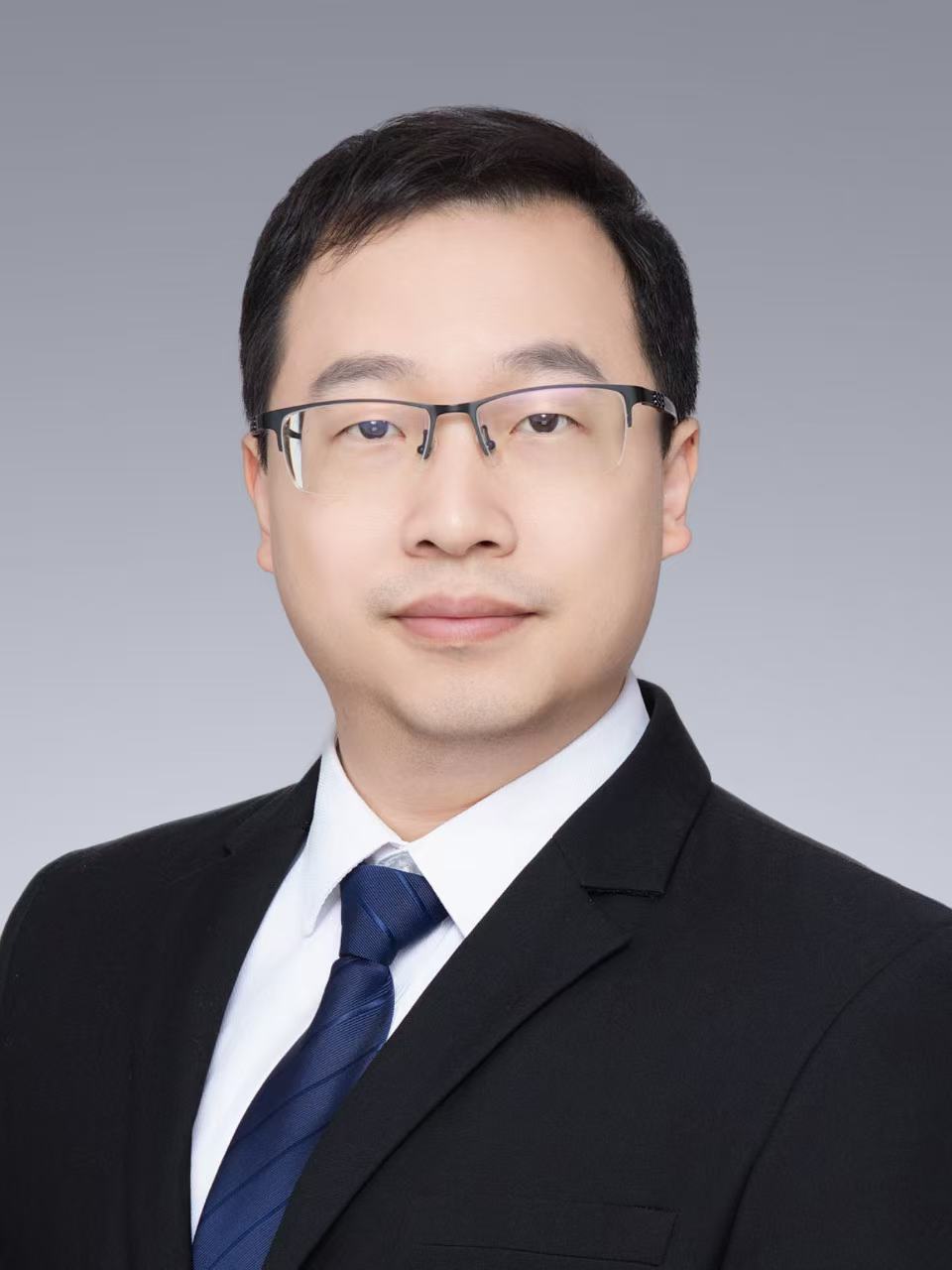}}]
{Shao-Lun Huang}
(Member, IEEE) received the B.S. degree (Hons.) from the Department of Electronic Engineering, National Taiwan University, Taipei, Taiwan, in 2008, and the M.S. and Ph.D. degrees from the Department of Electronic Engineering and Computer Sciences, Massachusetts Institute of Technology, Cambridge, MA, USA, in 2010 and 2013, respectively. From 2013 to 2016, he was a Post-Doctoral Researcher with the Department of Electrical Engineering, National Taiwan University, and also with the Department of Electrical Engineering and Computer Science, Massachusetts Institute of Technology. Since 2016, he has been with the Tsinghua-Berkeley Shenzhen Institute and Tsinghua Shenzhen International Graduate School, Tsinghua University, China, where he is currently an Associate Professor. His current research interests include information theory, communication theory, machine learning, and social networks.
\end{IEEEbiography}

\clearpage
\appendices

\section{Justification of the KL-Based Measure}
\label{appendix_K_L}

Compared to other measures, the K-L divergence exhibits a closer correspondence with the generalization error as measured by the cross-entropy loss. In this section, we will provide a concrete example (i.e., in a classification task) where the proposed K-L divergence formulation measure aligns with the cross-entropy generalization error.

In a classification task, let $ z \in \mathcal{Z} $ denote the input features and $ y \in \mathcal{Y} $ represent the output labels. The true data-generating distribution is $ P (z, y) $, while the joint distribution model learned from the training dataset is denoted as $ \hat{P} _ {\hat{\theta}} (z, y) $, where $\hat{\theta}$ denotes the learnable model parameters in training. The proposed measure $D(P \parallel \hat{P})$ is defined as follows.

\begin{align}
\label{K-Ldecomposition}
&D(P \parallel \hat{P}) =\sum_{z \in \mathcal{Z},y \in \mathcal{Y}} P(z, y) \log \frac{P(z, y)}{\hat{P} _ {\hat{\theta}} (z, y)}
\notag\\&=\sum_{z \in \mathcal{Z},y \in \mathcal{Y}} P(z, y) \log P(z, y)-\sum_{z \in \mathcal{Z},y \in \mathcal{Y}}P(z, y) \log {\hat{P} _ {\hat{\theta}} (y \mid z)}\notag\\&-\sum_{z \in \mathcal{Z},y \in \mathcal{Y}}P(z, y) \log {\hat{P}(z)},
\end{align}
where the second term of \eqref{K-Ldecomposition}, i.e., $$-\sum_{z \in \mathcal{Z},y \in \mathcal{Y}}P(z, y) \log {\hat{P} _ {\hat{\theta}} (y \mid z)},$$ is the standard definition of generalization error in a classification task.
Furthermore, the first term of \eqref{K-Ldecomposition} depends solely on the true data distribution and is constant. Though the third term involves $\hat{P}(z)$, $\hat{P}(z)$ is typically treated as a fixed marginal distribution decided by the training data rather than being learned in the model training. As a result, the third term is usually not parameterized by model parameter $\hat{\theta}$, and thus does not contribute to the optimization objective. In brief, the first term and the third term will not affect our optimization. Therefore, the K-L divergence formulation remains consistent with the standard definition of generalization error.

\section{Proof of Theorem \ref{thm:one_source}}
\label{appendix:one_source}
\begin{proof}

\begin{lemma}\label{thm:kl2mse}

In the asymptotic case, the proposed measure \eqref{eq:KL1} and the mean squared error have the relation as follows.
\begin{align}\label{eq:kl2mse}
\mathbb{E} \left[ D\left(P_{X;\theta_0}\middle\| P_{X;\hat{\theta}}\right) \right]&=\frac{1}{2} {J}(\theta_0)  \mathbb{E} \left[\left(\hat{\theta} - \theta_0\right)^2 \right] +\notag\\&o\left( \mathbb{E} \left[\left(\hat{\theta} - \theta_0\right)^2 \right] \right)+o(\frac{1}{N_0}).
\end{align}
\end{lemma}
\begin{proof}
In this section, for the sake of clarity, we will write $ \hat{\theta}$ in its parameterized form $\hat{\theta}(X^{N_0})$ when necessary, and these two forms are mathematically equivalent.
By taking Taylor expansion of $D\left(P_{X;\theta_0}\middle\| P_{X;\hat{\theta}(X^{N_0})} \right)$ at $\theta_0$, we can get

\begin{align}
\label{appendix:taylor1}
&D\left( P_{X;\theta_0}\middle\| P_{X;\hat{\theta}(X^{N_0})}\right)
\notag\\
&=-\sum_{x\in X}P_{X;\theta_0}(x)\log\frac{P_{X;\hat{\theta}(X^{N_0})}(x)}{P_{X;\theta_0}(x)}
\notag\\&=-\sum_{x\in X}P_{X;\theta_0}(x)\log\left(1+\frac{P_{X;\hat{\theta}(X^{N_0})}(x)-P_{X;\theta_0}(x)}{{P_{X;\theta_0}(x)}}\right)
\notag\\&=-\sum_{x\in X}P_{X;\theta_0}(x)\Bigg(\left(\frac{P_{X;\hat{\theta}(X^{N_0})}(x)-P_{X;\theta_0}(x)}{{P_{X;\theta_0}(x)}}\right)
\notag\\&-\frac{1}{2}\bigg(\frac{P_{X;\hat{\theta}(X^{N_0})}(x)-P_{X;\theta_0}(x)}{{P_{X;\theta_0}(x)}}\bigg)^2\Bigg)+o(\vert\hat{\theta}(X^{N_0}) - \theta_0\vert^2)  
\notag\\&= \frac{1}{2} \sum_{x\in X} \frac{\left(P_{X;\hat{\theta}(X^{N_0})}(x) - P_{X;\theta_0}(x)\right)^2}{P_{X;\theta_0}(x)}+o(\vert\hat{\theta}(X^{N_0}) - \theta_0\vert^2)  
\end{align}

                                                                                           
We denote $\delta$ as a small constant,  and we can rewrite \eqref{eq:KL1} as 
\begin{align}
\label{appendix:expect}
&\mathbb{E} \left[ D\left(P_{X;\theta_0}\middle\| P_{X;\hat{\theta}(X^{N_0})} \right) \right] \notag\\ 
&= \sum_{X^{N_0}}P_{X^n;\theta_0}(X^{N_0})D\left(P_{X;\theta_0}\middle\| P_{X;\hat{\theta}(X^{N_0})} \right)
\notag\\ 
&= \sum_{\left\{X^{N_0}:\vert\hat{\theta}(X^{N_0}) - \theta_0\vert^2<\delta\right\}}P_{X^n;\theta_0}(X^{N_0}) \cdot\notag\\ 
    &\left ( \frac{1}{2} \sum_{x\in X} \frac{\left(P_{X;\hat{\theta}(X^{N_0})}(x) - P_{X;\theta_0}(x)\right)^2}{P_{X;\theta_0}(x)}+o(\vert\hat{\theta}(X^{N_0}) - \theta_0\vert^2)  \right)\notag\\ 
    &+\sum_{\left\{X^{N_0}:\vert\hat{\theta}(X^{N_0}) - \theta_0\vert^2\ge\delta\right\}}P_{X^n;\theta_0}(X^{N_0})\cdot
    D\left(P_{X;\theta_0}\middle\| P_{X;\hat{\theta}(X^{N_0})} \right)\notag\\ 
    &= \sum_{\left\{X^{N_0}:\vert\hat{\theta}(X^{N_0}) - \theta_0\vert^2<\delta\right\}}P_{X^n;\theta_0}(X^{N_0}) \cdot\notag\\   
    &\left ( \frac{1}{2} \sum_{x\in X} \frac{ \left(\frac{\partial P_{X;\theta_0}(x)}{\partial \theta_0} (\hat{\theta}(X^{N_0}) - \theta_0)\right)^2}{P_{X;\theta_0}(x)}+o(\vert\hat{\theta}(X^{N_0}) - \theta_0\vert^2)  \right)\notag\\ 
    &+\sum_{\left\{X^{N_0}:\vert\hat{\theta}(X^{N_0}) - \theta_0\vert^2\ge\delta\right\}}P_{X^n;\theta_0}(X^{N_0})\cdot D\left(P_{X;\theta_0}\middle\| P_{X;\hat{\theta}(X^{N_0})} \right)\notag\\   
\end{align}
To facilitate the subsequent proof, we introduce the concept of "Dot Equal".
\begin{definition}(Dot Equal (\(\dot{=}\)))
    Specifically, given two functions \( f(n) \) and \( g(n) \), the notation \( f(n) \dot{=} g(n) \) is defined as
\begin{align}
    f(n) \dot{=} g(n) \quad \Leftrightarrow \quad \lim_{n \to \infty} \frac{1}{n}\log\frac{f(n)}{g(n)} = 0,
\end{align}
which shows that $f(n)$ and $g(n)$ have the same exponential decaying rate.
\end{definition}

We denote $\hat{P}_{X^{N_0}}$ as the empirical distribution of $X^{N_0}$. Applying Sanov's Theorem to \eqref{appendix:expect}, we can know that
\begin{align}
\label{appendix:dotequal_1}
P_{X^n;\theta_0}(X^{N_0})\doteq e^{-N_{0}D\left(\hat{P}_{X^{N_0}} \middle\| P_{X;\theta_0}\right)} 
\end{align}
Then, we aim to establish a connection between 
\eqref{appendix:dotequal_1} and $\vert\hat{\theta}(X^{N_0}) - \theta_0\vert^2$. From \eqref{appendix:taylor1}, we can know that the $D\left(\hat{P}_{X^{N_0}} \middle\| P_{X;\theta_0}\right)$ in \eqref{appendix:dotequal_1} can be transformed to
\begin{align}
\label{appendix:taylor2}
&D\left( \hat{P}_{X^{N_0}}\middle\| P_{X;\theta_0}\right)\notag\\&=\frac{1}{2} \sum_{x\in X} \frac{\left(\hat{P}_{X^{N_0}}(x) - P_{X;\theta_0}(x)\right)^2}{\hat{P}_{X^{N_0}}(x)}+o(\vert\hat{\theta}(X^{N_0}) - \theta_0\vert^2)  
\notag\\&= \frac{1}{2} \sum_{x\in X} \frac{\left(\hat{P}_{X^{N_0}}(x) - P_{X;\theta_0}(x)\right)^2}{P_{X;\theta_0}(x)}+o(\vert\hat{\theta}(X^{N_0}) - \theta_0\vert^2)   
\end{align}

From the characteristics of MLE, we can know that
\begin{align}
\label{appendix:mle_a1}
    &\mathbb{E}_{\hat{P}_{X^{N_0}}}\left[\frac{\partial\log P_{X;{\hat{\theta}(X^{N_0})}}(x)}{\partial \hat{\theta}}\right]
    \notag\\&=0   \notag\\&=\mathbb{E}_{\hat{P}_{X^{N_0}}}\left[\frac{\partial\log P_{X;{\theta_{0}}}(x)}{\partial \theta_{0}}\right]\notag\\&+\mathbb{E}_{\hat{P}_{X^{N_0}}}\left[\frac{\partial^{2}\log P_{X;{\theta_{0}}}(x)}{\partial \theta_0^{2}}\right]{\left(\hat{\theta}(X^{N_0})-\theta_0\right)}\notag\\&+O(\vert\hat{\theta}(X^{N_0}) - \theta_0\vert^2)  ,
\end{align}
which can be transformed to
\begin{align}
\label{appendix:mle_a2}
&\left(\hat{\theta}(X^{N_0})-\theta_0\right)+O(\vert\hat{\theta}(X^{N_0}) - \theta_0\vert^2)
\notag\\&=-\frac{\mathbb{E}_{\hat{P}_{X^{N_0}}}\left[\frac{\partial\log P_{X;{\theta_{0}}}(x)}{\partial \theta_0}\right]}{\mathbb{E}_{\hat{P}_{X^{N_0}}}\left[\frac{\partial^{2}\log P_{X;{\theta_{0}}}(x)}{\partial \theta_0^{2}}\right]}
\notag\\&=-\frac{\mathbb{E}_{\hat{P}_{X^{N_0}}}\left[\frac{\frac{\partial P_{X;{\theta_{0}}}(x)}{\partial \theta_0}}{P_{X;{\theta_{0}}}(x)}\right]}{\mathbb{E}_{\hat{P}_{X^{N_0}}}\left[\frac{\partial^{2}\log P_{X;{\theta_{0}}}(x)}{\partial \theta_0^{2}}\right]}
\notag\\&=\frac{\sum\limits_{x\in\mathcal{X}}{\left(\hat{P}_{X^{N_0}}(x)- P_{X;\theta_0}(x)\right)}\frac{\frac{\partial P_{X;{\theta_{0}}}(x)}{\partial \theta_0}}{P_{X;{\theta_{0}}}(x)}}{J(\theta_{0})}
\end{align}

Using the Cauchy-Schwarz inequality, we can obtain
\begin{align}
\label{Cauchy}
&\sum_{x\in X} \frac{\left(\hat{P}_{X^{N_0}}(x) - P_{X;\theta_0}(x)\right)^2}{P_{X;\theta_0}(x)}\cdot \sum_x \frac{(\frac{\partial P_{X;\theta_0}(x)}{\partial \theta_0})^2}{P_{X;\theta_0}(x)}
\ge
\notag\\&\left(\sum\limits_{x\in\mathcal{X}}{\left(\hat{P}_{X^{N_0}}(x)- P_{X;\theta_0}(x)\right)}\frac{\frac{\partial P_{X;{\theta_{0}}}(x)}{\partial \theta_0}}{P_{X;{\theta_{0}}}(x)}\right)^2
\end{align}
where
\begin{align}
\label{partial_to_fisher}
       &\sum_x \frac{(\frac{\partial P_{X;\theta_0}(x)}{\partial \theta_0})^2}{P_{X;\theta_0}(x)}
       =\sum_x P_{X;\theta_0}(x) \left( \frac{1}{P_{X;\theta_0}(x)} \frac{\partial P_{X;\theta_0}(x)}{\partial \theta_0} \right)^2 
       \notag\\&=\sum_x P_{X;\theta_0}(x) \left( \frac{\partial \log P_{X;\theta_0}(x)}{\partial \theta_0} \right)^2 
       =J(\theta_0) 
\end{align}
Combining with \eqref{appendix:taylor2}, \eqref{appendix:mle_a2}, and \eqref{partial_to_fisher}, the inequality \eqref{Cauchy} can be transformed to 
\begin{align}
\label{Cauchy2}
&D\left( \hat{P}_{X^{N_0}}\middle\|P_{X;\theta_0} \right)
\notag\\&= \frac{1}{2} \sum_{x\in X} \frac{\left(\hat{P}_{X^{N_0}}(x) - P_{X;\theta_0}(x)\right)^2}{P_{X;\theta_0}(x)}+o(\vert\hat{\theta}(X^{N_0}) - \theta_0\vert^2)  \notag\\&\ge\frac{1}{2}J(\theta_0)\left(\hat{\theta}(X^{N_0})-\theta_0+O(\vert\hat{\theta}(X^{N_0}) - \theta_0\vert^2)\right)^2\notag\\&+o(\vert\hat{\theta}(X^{N_0}) - \theta_0\vert^2)  
\notag\\&=\frac{1}{2}J(\theta_0)\vert\hat{\theta}(X^{N_0}) - \theta_0\vert^2
+o(\vert\hat{\theta}(X^{N_0}) - \theta_0\vert^2)  
\end{align}

Combining \eqref{appendix:dotequal_1} and \eqref{Cauchy2}, we can know that
\begin{align}
\label{appendix:dotequal_2}
    P_{X^n;\theta_0}(X^{N_0})\doteq e^{-N_{0}D\left(\hat{P}_{X^{N_0}} \middle\| P_{X;\theta_0}\right)} \leq e^{\frac{-N_{0}J(\theta_0)\vert\hat{\theta}(X^{N_0}) - \theta_0\vert^2}{2}} 
\end{align}

For the first term in \eqref{appendix:expect}, the negligible term $o(\vert\hat{\theta}(X^{N_0}) - \theta_0\vert^2)$ is of lower order than the main term and can therefore be ignored. 
As for the second term in \eqref{appendix:expect}, its probability 
is bounded by 
$O( e^{\frac{-N_{0}J(\theta_0)\vert\hat{\theta}(X^{N_0}) - \theta_0\vert^2}{2}})$
according to \eqref{appendix:dotequal_2}, which decays 
exponentially with $N_0$. Consequently, this probability is $o(1/N_0)$, since exponential 
decay in $N_0$ dominates any polynomial rate. Moreover, under the compactness assumption on the 
parameter space stated in Section~\ref{subsection_Problem Formulation}, the K-L divergence is 
uniformly bounded, implying that the entire second term is 
$o(1/N_0)$.
By transferring \eqref{appendix:expect}, we can get
\begin{align}
\label{appendix:expect3}
&\mathbb{E} \left[ D\left(P_{X;\theta_0}\middle\| P_{X;\hat{\theta}(X^{N_0})} \right) \right] \notag\\ 
&=\frac{1}{2} \mathbb{E} \left[ \sum_{x\in X} \frac{ \left(\frac{\partial P_{X;\theta_0}(x)}{\partial \theta_0} (\hat{\theta}(X^{N_0}) - \theta_0)\right)^2}{P_{X;\theta_0}(x)} \right]\notag\\&+o\left( \mathbb{E} \left[\left(\hat{\theta} - \theta_0\right)^2 \right] \right)+o{(\frac{1}{N_0})}.
\end{align}


We then transform \eqref{appendix:expect3} with \eqref{partial_to_fisher}
\begin{align}
    \label{KL1_cont}
    &\frac{1}{2} \mathbb{E} \left[  \sum_{x\in X} \frac{ \left(\frac{\partial P_{X;\theta_0}(x)}{\partial \theta_0} (\hat{\theta}(X^{N_0}) - \theta_0)\right)^2}{P_{X;\theta_0}(x)} \right]\notag\\
    &=  \frac{1}{2} \mathbb{E} \left[\left(\hat{\theta} - \theta_0\right)^2 \right] \sum_x \frac{(\frac{\partial P_{X;\theta_0}(x)}{\partial \theta_0})^2}{P_{X;\theta_0}(x)}  \notag\\
    &= \frac{1}{2} \mathbb{E} \left[\left(\hat{\theta} - \theta_0\right)^2 \right]  {J}(\theta_0) 
\end{align}
Combining  \eqref{appendix:expect3}, \eqref{KL1_cont}, we can get 
\begin{align}
\mathbb{E} \left[ D\left(P_{X;\theta_0}\middle\| P_{X;\hat{\theta}}\right) \right]&=\frac{1}{2} \mathbb{E} \left[\left(\hat{\theta} - \theta_0\right)^2 \right]  {J}(\theta_0) +\notag\\&o\left( \mathbb{E} \left[\left(\hat{\theta} - \theta_0\right)^2 \right] \right)+o{(\frac{1}{N_0})}.
\end{align}
\end{proof}

We denote $E_{\hat{\theta}}$ as the expectation of $\hat{\theta}$. Then, we can make a decomposition of the K-L measure, i.e.,
\begin{align}
\label{KLdecomposition}
& \mathbb{E}\left[D\left( P_{X;\theta_0}\middle\| P_{X;\hat{\theta}}\right)\right]
\notag\\
&= \frac{1}{2} \mathbb{E} \left[\left(\hat{\theta} - \theta_0\right)^2 \right]  {J}(\theta_0)+o(\frac{1}{N_0}) \notag\\
&=\frac{1}{2}{J}(\theta_0)\left(\mathbb{E}\left[ \left(\hat{\theta} - E_{\hat{\theta}}\right)^2\right]+\mathbb{E}\left[ \left(E_{\hat{\theta}} - \theta_0\right)^2\right]\right)
\notag\\&+o\left(\mathbb{E}\left[ \left(\hat{\theta} - E_{\hat{\theta}}\right)^2\right]+\mathbb{E}\left[ \left(E_{\hat{\theta}} - \theta_0\right)^2\right]\right)+o(\frac{1}{N_0}).
\end{align}

In the following, we first derive the expression of $E_{\hat{\theta}}$, as well as the distributional expression of $(\hat{\theta} - E_{\hat{\theta}})$. We then use these results to obtain the representation of the K-L measure.
In our setting, the MLE and its expected definition are given as follows.

\begin{align}
    \hat{\theta} \defeq \argmax_{\theta} L_n(\theta),\label{eq_theta} \\ 
    E_{\hat{\theta}} \defeq \argmax_{\theta} L(\theta) .\label{eq_E_theta}
\end{align}
where
\begin{align}
    &L_n(\theta) = \frac{\sum\limits_{x \in X^{N_0}} \log P_{X;\theta}(x) + \sum\limits_{x \in X^{n_1}} w_1 \log P_{X;\theta}(x)}{N_0 + w_1n_1}, \label{L_function_expection}\\ 
    &L(\theta) = \frac{N_0\mathbb{E}_{\theta_0}[ \log P_{X;\theta}(x)]+w_1n_1\mathbb{E}_{\theta_1}[ \log P_{X;\theta}(x)]}{N_0 + w_1n_1}. 
\end{align}

For $E_{\hat{\theta}}$, we have the following lemma to establish its relation to the $\theta_0$ and $\theta_1$.
\begin{lemma}
\label{lemma_E_theta}$E_{\hat{\theta}}$ is a weighted average of $\theta_0$ and $\theta_1$,  given by

\begin{align}
\label{weighted_theta}
    E_{\hat{\theta}} = \frac{N_0 \theta_0 + w_1n_1 \theta_1}{N_0 + w_1n_1} + O\left(\frac{1}{N_{0}}\right).
\end{align}
\end{lemma}
\begin{proof}
We could equivalently transform \eqref{eq_E_theta} into
\begin{align}
     E_{\hat{\theta}} \defeq \argmin_{\theta} -L(\theta).
\end{align}
We then employ the Lagrangian method to solve this minimization problem. We treat $P_{X;\theta}(x), \forall x\in \mathcal{X}$ as the variables with constraint 
\begin{align}\label{sumequal1}\sum\limits_{x\in \mathcal{X}} P_{X;\theta}(x) = 1 .
\end{align} Then we can construct the corresponding Lagrangian function

\begin{align}
&\text{Lagrangian}(P, \lambda) =
- \sum_{x\in \mathcal{X}} \bigg( \frac{N_0P_{X;\theta_0}(x)}{N_0 + w_1n_1} +
\notag\\&
\frac{w_1n_1P_{X;\theta_1}(x)}{N_0 + w_1n_1}  \bigg) \log P_{X;\theta}(x)
+ \lambda \left( \sum_{x\in \mathcal{X}} P_{X;\theta}(x) - 1 \right).
\end{align}

By jointly solving the first-order necessary conditions
\begin{align}
&\frac{\partial\text{Lagrangian}(P, \lambda)}{\partial P_{X;\theta}(x)}=0, \forall x\in\mathcal{X}\\
&\frac{\partial\text{Lagrangian}(P, \lambda)}{\partial\lambda}=0,
\end{align}
we have
\begin{align}
\label{eq:beginprove37}
P_{X;E_{\hat{\theta}}}(x) = \frac{N_0P_{X;\theta_0}(x)+w_1n_1P_{X;\theta_1}(x)}{N_0 + w_1n_1} , \forall x\in \mathcal{X}.
\end{align}
By doing a Taylor Expansion of \eqref{eq:beginprove37} around $\theta_0$, we can get 
\begin{align}
\label{eq:1111}
&P_{X;\theta_0}(x)+\frac{\partial P_{X;\theta_0}(x)}{\partial \theta}(E_{\hat{\theta}}-\theta_0)+O(\vert E_{\hat{\theta}}-\theta_0\vert^2)\notag\\&=\frac{N_0}{N_0 + w_1n_{1}} P_{X;\theta_0}(x)\notag +\frac{w_1n_{1}}{N_0 + w_1n_{1}} \bigg(P_{X;\theta_0}(x)+\\ &\frac{\partial P_{X;\theta_0}(x)}{\partial \theta}(\theta_1-\theta_0)+O(\vert\theta_0-\theta_1\vert^2)\bigg).
\end{align}
From \eqref{eq:1111} we can get
\begin{align}
\label{eq:3737}
    E_{\hat{\theta}} = \frac{N_{0}\theta_0 + w_1n_{1}\theta_1}{N_{0}+w_1n_{1}} + O(\vert\theta_0-\theta_1\vert^2),
\end{align}
 
So we can get 
\begin{align}
\label{eq:373737}
    E_{\hat{\theta}} = \frac{N_{0}\theta_0 + w_1n_{1}\theta_1}{N_{0}+w_1n_1} + O\left(\frac{1}{N_{0}}\right).
\end{align}
\end{proof}

Because $\hat{\theta}$ is a maximizer of $L_n(\hat{\theta})$, $E_{\hat{\theta}}$ is a maximizer of $L(\hat{\theta})$, we can get
\begin{align}
    &L_n'(\hat{\theta}) = 0 = L_n'(E_{\hat{\theta}}) + L_n''(E_{\hat{\theta}})(\hat{\theta} - E_{\hat{\theta}})
    , \label{derivative1to2} \\
    &L'(E_{\hat{\theta}}) = 0 = \frac{N_0}{N_0 + w_1n_1} \mathbb{E}_{\theta_0}\left[ \frac{\partial\log{P_{X;E_{\hat{\theta}}}(x)}}{\partial \theta}\right] 
    \notag \\&\qquad\qquad\quad
    + \frac{w_1n_1}{N_0 + w_1n_1} \mathbb{E}_{\theta_1}\left[\frac{\partial\log{P_{X;E_{\hat{\theta}}}(x)}}{\partial \theta}\right]\label{derivative1to2_2}. 
\end{align}

By transforming \eqref{derivative1to2}, we have
\begin{align}
\label{frac_distribution}
    &\sqrt{N_0 + w_1n_1} (\hat{\theta} - E_{\hat{\theta}}) = -\frac{\sqrt{N_0 + w_1n_1}L_n'(E_{\hat{\theta}})}{L_n''(E_{\hat{\theta}})} . 
\end{align}
Combining with \eqref{derivative1to2_2}, we can transform the numerator of the right-hand side of \eqref{frac_distribution} as follows:

\begin{align}
\label{eq21_2025_6_27}
    &\sqrt{N_0 + w_1n_1} L_n'(E_{\hat{\theta}}) 
    \notag \\
    & = \sqrt{N_0 + w_1n_1} \left(L_n'(E_{\hat{\theta}}) - L'(E_{\hat{\theta}})\right)
    \notag \\
    & = \sqrt{\frac{N_0}{N_0 + w_1n_1}} \Bigg( \sqrt{\frac{1}{N_0}} \sum_{x \in X^{N_0}} \frac{\partial \log{P_{X;E_{\hat{\theta}}}(x)}}{\partial \theta} 
    \notag \\&
    - \sqrt{N_0} \mathbb{E}_{\theta_0}[\frac{\partial \log{P_{X;E_{\hat{\theta}}}(x)}}{\partial \theta}] \Bigg) \notag \\ 
    &+\sqrt{\frac{w_1n_1}{N_0 + w_1n_1}}\Bigg( \sqrt{\frac{1}{w_1n_1}} \sum_{x \in X^{n_1}} w_1 \frac{\partial \log{P_{X;E_{\hat{\theta}}}(x)}}{\partial \theta}  
    \notag \\ &
    - \sqrt{w_1n_1} \mathbb{E}_{\theta_1}[\frac{\partial \log{P_{X;E_{\hat{\theta}}}(x)}}{\partial \theta}] \Bigg).
\end{align}

Applying the Central Limit Theorem to \eqref{eq21_2025_6_27}, we can get
\begin{align}\label{lntheta3}
    &\sqrt{N_{0}+w_1n_1} L_n^{'}( E_{\hat{\theta}})\xrightarrow{a.s.}\\&\mathcal{N} \Bigg(0, 
    \frac{N_0}{N_{0}+w_1n_1}\bigg(\mathbb{E}_{\theta_0}\left[\left ( \frac{\partial \log{P_{X;E_{\hat{\theta}}}(x)}}{\partial \theta}  \right )^2 \right]  
    \nonumber\\&-\mathbb{E}_{\theta_0}\left[\left ( \frac{\partial \log{P_{X;E_{\hat{\theta}}}(x)}}{\partial \theta}  \right ) \right]^2\bigg)+\nonumber\\
    &\frac{w_1^2n_1}{N_{0}+w_1n_1}\bigg(\mathbb{E}_{\theta_1}\left[\left ( \frac{\partial \log{P_{X;E_{\hat{\theta}}}(x)}}{\partial \theta}  \right )^2 \right]
    \nonumber\\&-\mathbb{E}_{\theta_1}\left[\left ( \frac{\partial \log{P_{X;E_{\hat{\theta}}}(x)}}{\partial \theta}  \right ) \right]^2\bigg) \Bigg)
\end{align}

By taking Taylor expansion of $ E_{\hat{\theta}}$ at $\theta_0$, we can get
\begin{align}
\label{lntheta31}
&\mathbb{E}_{\theta_0}\left[\left ( \frac{\partial \log{P_{X;E_{\hat{\theta}}}(x)}}{\partial \theta}  \right )^2\right]\nonumber\\
&=\mathbb{E}_{\theta_0}\bigg[\bigg( \frac{\partial \log{P_{X;\theta_0}(x)}}{\partial \theta}+\frac{\partial }{\partial \theta} \frac{\partial \log{P_{X;\theta_0}(x)}}{\partial  \theta}( E_{\hat{\theta}}-\theta_0)
\nonumber\\&+O(\frac{1}{N_{0}}) \bigg)^2\bigg]\nonumber\\
&=J(\theta_0)+( E_{\hat{\theta}}-\theta_0)\mathbb{E}_{\theta_0}\left [ \frac{\partial \log{P_{X;E_{\hat{\theta}}}(x)}}{\partial \theta}\frac{\partial^2  \log{P_{X;E_{\hat{\theta}}}(x)}}{\partial\theta^2 } \right ]\nonumber\\&+O(\frac{1}{N_{0}})\nonumber\\
&=J(\theta_0)+O(\frac{1}{\sqrt{N_{0}}})
\end{align}
and
\begin{align}
\label{lntheta32}
&\mathbb{E}_{\theta_0}\left[\left ( \frac{\partial \log{P_{X;E_{\hat{\theta}}}(x)}}{\partial \theta}  \right )\right]^2\nonumber\\
&=\mathbb{E}_{\theta_0}\bigg[\bigg ( \frac{\partial\log{P_{X;\theta_0}(x)} }{\partial \theta}+\frac{\partial }{\partial \theta} \frac{\partial \log{P_{X;\theta_0}(x)}}{\partial  \theta}( E_{\hat{\theta}}-\theta_0)\nonumber\\&+O(\frac{1}{N_{0}}) \bigg )\bigg]^2\nonumber\\
&=\mathbb{E}_{\theta_0}\left[\left ( \frac{\partial }{\partial \theta} \frac{\partial \log{P_{X;\theta_0}(x)}}{\partial  \theta}( E_{\hat{\theta}}-\theta_0)+O(\frac{1}{N_{0}}) \right )\right]^2\nonumber\\
&=( E_{\hat{\theta}}-\theta_0)^2E_{\theta_0}\left[\left ( \frac{\partial }{\partial \theta} \frac{\partial \log{P_{X;\theta_0}(x)}}{\partial  \theta} \right )\right]^2+o(\frac{1}{N_{0}})\nonumber\\&=O(\frac{1}{N_{0}})
\end{align}

Similarly to \eqref{lntheta31} and \eqref{lntheta32}, we can prove
\begin{align}
\label{lntheta33}
&\mathbb{E}_{\theta_1}\left[\left ( \frac{\partial \log{P_{X;E_{\hat{\theta}}}(x)}}{\partial \theta}  \right )^2\right]=J(\theta_1)+O(\frac{1}{\sqrt{N_{0}}}),
\end{align}
and
\begin{align}
\label{lntheta34}
&\mathbb{E}_{\theta_1}\left[\left ( \frac{\partial \log{P_{X;E_{\hat{\theta}}}(x)}}{\partial \theta}  \right )\right]^2=O(\frac{1}{N_{0}}).
\end{align}

By substituting \eqref{lntheta31}, \eqref{lntheta32}, \eqref{lntheta33}, and \eqref{lntheta34} into \eqref{lntheta3}, we obtain
\begin{align}
\label{CLT}
    &\sqrt{N_0 + w_1n_1} L_n'(E_{\hat{\theta}}) \xrightarrow{a.s.}
    \notag\\&
    \mathcal{N}\Bigg(0, \frac{N_0}{N_0 + w_1n_1} J(\theta_0) + \frac{w_1^2n_1}{N_0 + w_1n_1}  J(\theta_1)\Bigg) .
\end{align}

Moreover, we know $L_n^{''}{( E_{\hat{\theta}})}\xrightarrow{p}-J( E_{\hat{\theta}})$. Given that $\theta_0$, $\theta_1$, and $E_{\hat{\theta}}$ are already assumed to be sufficiently close, we can use this, along with a Taylor expansion, to estimate the distance between their Fisher information matrices. We conclude that the discrepancy among ${J}(\theta_0)$, ${J}(\theta_1)$, and ${J}(E_{\hat{\theta}})$ is of the order $O\left(\frac{1}{\sqrt{N_0}}\right)$.
\begin{align}
\label{appendix_Fisher_matrixs_relations}
    J({\theta}_1) &= \mathbb{E}_{{\theta}_1} \bigg[ \left( \frac{\partial}{\partial {{\theta}}} \log P_{X;{{\theta}_1}} \right) ^2 \bigg]\notag\\
    &=\mathbb{E}_{{\theta}_1} \bigg[ \bigg( \frac{\partial}{\partial {{\theta}}} \log P_{X;{{\theta}_0}}+\frac{\partial^2\log P_{X;{{\theta}_0}}}{\partial {{\theta}^2}}({\theta}_1-{\theta}_0)\notag\\ 
    &+O(\frac{1}{N_0}) \bigg) ^2 \bigg]\notag\\
    &=\mathbb{E}_{{\theta}_1} \bigg[ \left( \frac{\partial}{\partial {{\theta}}} \log P_{X;{{\theta}_0}} \right) ^2 \bigg]+O(\frac{1}{\sqrt{N_0}})\notag\\
    &=\sum_{x\in \cX}P_{X;{{\theta}_1}}(x)\left( \frac{\partial}{\partial {{\theta}}} \log P_{X;{{\theta}_0}} (x)\right) ^2+O(\frac{1}{\sqrt{N_0}}) \notag\\
    &=\sum_{x\in \cX}\left(P_{X;{{\theta}_0}}(x)+\frac{\partial P_{X;{{\theta}_0}}}{\partial {{\theta}}}({\theta}_1-{\theta}_0)+O(\frac{1}{N_0})\right)\cdot\\&\left( \frac{\partial}{\partial {{\theta}}} \log P_{X;{{\theta}_0}} (x)\right) ^2+O(\frac{1}{\sqrt{N_0}}) \notag \\
    &=\sum_{x\in \cX}P_{X;{{\theta}_0}}(x)\left( \frac{\partial}{\partial {{\theta}}} \log P_{X;{{\theta}_0}} (x)\right) ^2+O(\frac{1}{\sqrt{N_0}}) \\
    &=J({\theta}_0)+O(\frac{1}{\sqrt{N_0}}) 
\end{align}

Combining with \eqref{frac_distribution} and \eqref{CLT},
\begin{align}
\label{CLT2}
    (\hat{\theta} - E_{\hat{\theta}})\xrightarrow{d} N\left(0, \frac{N_0 + w_1^2n_1}{(N_0 + w_1n_1)^2} \frac{1}{J(\theta_0)}\right).
\end{align}

Substituting \eqref{weighted_theta} and \eqref{CLT2} into \eqref{KLdecomposition}, we have
\begin{align}
\label{single_source_final}
    &\mathbb{E}\left[D\left( P_{X;\theta_0}\middle\| P_{X;\hat{\theta}}\right)\right]\notag \\
    &= \frac{1}{2} \frac{N_0 + w_1^2n_1}{(N_0+w_1n_1)^2} + \frac{1}{2}\frac{(w_1n_1)^2J(\theta_0)(\theta_0 - \theta_1)^2}{(N_0+w_1n_1)^2} +o(\frac{1}{N_0}).
\end{align}

By differentiating with respect to $ w_1 $, we obtain the optimal value of weight is
\begin{align}
    w_1^*=\frac{1}{1+tn_1}
\end{align}
Similarly, by taking the derivative of \eqref{single_source_final} with respect to $ n_1 $, we observe that the derivative is 
\begin{align}
    -\frac{1}{2}\left(\frac{1+n_1t}{N_0+n_1+N_0n_1t}\right),
\end{align}
which is strictly negative, implying that the optimal value of $ n_1 $ is $ N_1 $. Finally, we obtain the optimal value of weight in \eqref{optimal_weight_singlesource}. 
\end{proof}

\section{Proof of Proposition \ref{prop:highdimension}\label{appendix:Proposition:one_source_highdimension}}
\begin{proof}
    Similar to \eqref{KLdecomposition}, we can get
\begin{align}
\label{appendix:kl_to_twokl_prop2}
    & \mathbb{E} \left[ D\left( P_{X;{\underline{\theta}}_0}\middle\| P_{X;\hat{{\underline{\theta}}}}\right) \right] 
    \nonumber\\
    &=\frac{1}{2}tr\left({J}({\underline{\theta}}_0)\mathbb{E}\left[ \left(\hat{{\underline{\theta}}} - E_{\hat{{\underline{\theta}}}}\right)\left(\hat{{\underline{\theta}}} - E_{\hat{{\underline{\theta}}}}\right)^{\top}\right]\right)
    \nonumber\\&
    +\frac{1}{2}tr\left({J}({\underline{\theta}}_0)\mathbb{E}\left[ \left(E_{\hat{{\underline{\theta}}}} - {\underline{\theta}}_0\right)\left(E_{\hat{{\underline{\theta}}}} - {\underline{\theta}}_0\right)^{\top}\right]\right)+o(\frac{1}{N_0}).
\end{align}
Similar to \eqref{CLT2},we can get
\begin{align}
\label{square_distribution_one_source}
\left(\hat{{\underline{\theta}}} - E_{\hat{{\underline{\theta}}}}\right)\xrightarrow{d}\mathcal{N}\left(0,\frac{N_0 + w_1^2n_1}{(N_0 + w_1n_1)^2} J({\underline{\theta}}_0)^{-1}\right).
\end{align}

So we can know that $\mathbb{E}\left[ \left(\hat{{\underline{\theta}}} - E_{\hat{{\underline{\theta}}}}\right)\left(\hat{{\underline{\theta}}} - E_{\hat{{\underline{\theta}}}}\right)^{\top}\right]$ has the limit
\begin{align}
\label{one_source_1_bound}
\frac{N_0 + w_1^2n_1}{(N_0 + w_1n_1)^2} J({\underline{\theta}}_0)^{-1}
\end{align}
The same to Lemma \ref{lemma_E_theta}, we can get
\begin{align}
\label{eq:theta3totheta12_prop}
    E_{\hat{{\underline{\theta}}}} = \frac{N_{0}{\underline{\theta}}_0 + w_1n_{1}{\underline{\theta}}_1}{N_{0}+w_1n_1} + O\left(\frac{1}{N_{0}}\right),
\end{align}
Combining \eqref{appendix:kl_to_twokl_prop2}, \eqref{square_distribution_one_source} and \eqref{eq:theta3totheta12_prop}, we can know that the K-L~measure has the form of \eqref{Proposition:one_source_eq_rewrite}.
\end{proof}

\section{Proof of Theorem \ref{thm:Weighted_multi_source}}
\label{appendix:Weighted_multi_source}
\begin{proof}
Similar to \eqref{CLT2},we can get
\begin{align}
\label{square_distribution_multi_source}
\left(\hat{{\underline{\theta}}} - E_{\hat{{\underline{\theta}}}}\right)\xrightarrow{d}
\mathcal{N}\left(0,\frac{N_0+\sum\limits_{i=1}^{K}w_i^2n_i}{\left(N_0+\sum\limits_{i=1}^{K}w_in_i\right)^2}J({\underline{\theta}}_0)^{-1}\right).
\end{align}
Similar to Lemma \ref{lemma_E_theta},we can get
\begin{align}
\label{multi_1_expectation}
E_{\hat{{\underline{\theta}}}}=\frac{N_0{\underline{\theta}}_0+\sum\limits_{i=1}^{k}w_in_i{\underline{\theta}}_i}{N_0+s}+O\left(\frac{1}{N_{0}}\right).
\end{align}


Combining \eqref{appendix:kl_to_twokl_prop2}, \eqref{square_distribution_multi_source} and \eqref{multi_1_expectation}, the K-L~measure is
\begin{align}
\label{appendix:eq:multisourcetarget1}
\frac{d}{2}\left(\frac{N_0+\sum\limits_{i=1}^{K}\frac{b_i^2}{n_i}}{(N_0+s)^2}+\frac{s^2}{(N_0+s)^2}\frac{\underline{\alpha}^T\Theta^T{J}(\underline{\theta}_0)\Theta\underline{\alpha}}{d}\right)+ o\left(\frac{1}{N_{0}}\right).\nonumber\\
\end{align}
According to (5),  it is easy to observe that the optimal performance is achieved when every
$n_i$ takes its maximum value $N_i$.  
Then, we can transform \eqref{appendix:eq:multisourcetarget1} to 
\begin{align}
&\frac{d}{2}\bigg(\frac{N_0}{(N_0+s)^2}+\frac{s^2}{(N_0+s)^2}\underline{\alpha}^T\operatorname{diag}\left(\frac{1}{N_1}, \dots, \frac{1}{N_K}\right)\underline{\alpha}\notag\\&+\frac{s^2}{(N_0+s)^2}\frac{\underline{\alpha}^T\Theta^T{J}(\underline{\theta}_0)\Theta\underline{\alpha}}{d}\bigg)+ o\left(\frac{1}{N_{0}}\right).\nonumber\\
\end{align}

\end{proof}

\end{document}